\documentclass{article} 
\usepackage{collas2023_conference,times}
\usepackage{easyReview}
\usepackage[utf8]{inputenc} 
\usepackage[T1]{fontenc}    
\usepackage{url}            
\usepackage{booktabs}       
\usepackage{amsfonts}       
\usepackage{amsmath}
\usepackage{nicefrac}       
\usepackage{microtype}      
\usepackage{xcolor}         
\usepackage{mathtools}
\usepackage{graphicx}
\usepackage{subfigure}
\usepackage{comment}
\usepackage{enumitem}

\usepackage{amsmath,amsfonts,bm}

\def\eqref#1{equation~\ref{#1}}

\def\1{\bm{1}}

\def\vc{{\bm{c}}}

\DeclareMathAlphabet{\mathsfit}{\encodingdefault}{\sfdefault}{m}{sl}
\SetMathAlphabet{\mathsfit}{bold}{\encodingdefault}{\sfdefault}{bx}{n}

\DeclareMathOperator*{\argmax}{arg\,max}

\usepackage[algo2e]{algorithm2e}
\usepackage{algorithm, algorithmic}
\usepackage{tablefootnote}
\usepackage{xspace}
\usepackage{wrapfig}
\usepackage{caption}
\usepackage{amsthm}

\newcounter{example}[section]
\newtheorem{proposition}{Proposition}
\newenvironment{example}[1][]{\refstepcounter{example}\par\medskip
  \noindent \textbf{Example~\theexample. #1} \rmfamily}{\medskip}

\usepackage{hyperref}
\hypersetup{
    colorlinks=true,
    linkcolor=red,
    filecolor=magenta,
    urlcolor=blue,
    citecolor=purple,
    pdftitle={Overleaf Example},
    pdfpagemode=FullScreen,
    }

\title{Dealing With Non-stationarity in Decentralized Cooperative Multi-Agent Deep Reinforcement Learning via Multi-Timescale Learning}

\author{
Hadi Nekoei\thanks{Correspondence to: \texttt{nekoeihe@mila.quebec.}} \\
Mila, Université de Montréal \\
\And
Akilesh Badrinaaraayanan \\
Mila, Université de Montréal \\
\And
Amit Sinha \\
Mila, University of McGill \\
\AND
Mohammad Amini \\
Mila, University of McGill \\
\And
Janarthanan Rajendran \\
Mila, Université de Montréal\\
\And 
Aditya Mahajan  \\
Mila, University of McGill  \\
\AND 
Sarath Chandar \\
Mila, Polytechnique Montréal \\
}

\collasfinalcopy 

\begin{document}

\maketitle

\begin{abstract}
Decentralized cooperative multi-agent deep reinforcement learning (MARL) can be a versatile learning framework, particularly in scenarios where centralized training is either not possible or not practical. One of the critical challenges in decentralized deep MARL is the non-stationarity of the learning environment when multiple agents are learning concurrently. A commonly used and efficient scheme for decentralized MARL is independent learning in which agents concurrently update their policies independently of each other.
We first show that independent learning does not always converge, while sequential learning where agents update their policies one after another in a sequence is guaranteed to converge to an agent-by-agent optimal solution. In sequential learning, when one agent updates its policy, all other agent's policies are kept fixed, alleviating the challenge of non-stationarity due to simultaneous updates in other agents' policies. However, it can be slow because only one agent is learning at any time. Therefore it might also not always be practical.
In this work, we propose a decentralized cooperative MARL algorithm based on multi-timescale learning. In multi-timescale learning, all agents learn simultaneously, but at different learning rates. In our proposed method, when one agent updates its policy, other agents are allowed to update their policies as well, but at a slower rate. This speeds up sequential learning, while also minimizing non-stationarity caused by other agents updating concurrently. Multi-timescale learning outperforms state-of-the-art decentralized learning methods on a set of challenging multi-agent cooperative tasks in the \textit{epymarl}~\citep{papoudakis2020benchmarking} benchmark. This can be seen as a first step towards more general decentralized cooperative deep MARL methods based on multi-timescale learning.
\end{abstract}

\section{Introduction}


In many emerging reinforcement learning (RL) applications, multiple agents interact in a shared environment. There are three types of such multi-agent environments:
(i)~Environments where agents are competitive and have an objective of maximizing their individual rewards: examples include games such as Poker~\citep{brown2018superhuman}, online auctions, and firms interacting in a market;
(ii)~Environments where agents are cooperative and have an objective of maximizing a common reward: examples include games such as Hanabi~\citep{bard2020hanabi}, multi-agent robotics~\citep{kober2013reinforcement}, networked control systems~\citep{yuksel2013stochastic}, power-grid systems~\citep{mai2023multi}, and self-driving cars;
(iii)~Environments where agents have mixed strategies and can be both cooperative and competitive: examples include games such as Football~\citep{kurach2019google} and Starcraft~\citep{vinyals2019grandmaster}. In this paper, we focus on cooperative multi-agent environments. 

When the system dynamics and reward function are known, cooperative multi-agent environments are studied using team theory~\citep{marshack1972economic}, cooperative game theory~\citep{Shapley2016}, or decentralized stochastic control~\citep{mahajan2013optimal}. When the system dynamics and reward function are unknown, they are investigated using multi-agent reinforcement learning (MARL).
One of the main challenges in cooperative MARL is the \emph{non-stationarity of the learning environment}. When multiple agents are learning and updating their policies concurrently, the transition dynamics and rewards are not stationary from a single agent's point of view since the next state of the environment is a function of the joint action of all agents and not only that agent's own action. 
The problem of non-stationarity becomes even more severe when the environment is partially observable which results in incomplete and asymmetric information across agents. 

One setting which is commonly used in the literature to circumvent the challenge of non-stationarity is to assume that agents are trained in an environment where a centralized critic can access the observations and actions of all agents (and potentially some or all components of the state). This centralized critic computes a centralized action-value function, which is then used by all agents to determine policies that can be executed in a decentralized manner by agents using just the local information available to them. This learning paradigm is called \emph{centralized training and decentralized execution} (CTDE). 
Although CTDE is able to circumvent the conceptual challenges of non-stationarity of the environment and partial observability, it is not an ideal solution in all scenarios. CTDE is only applicable when there is a centralized critic which has access to the observations and actions of all the agents. It is not always possible to construct such a centralized critic, especially in online real-world settings. For example, self-driving cars cannot share their policies and observations with other cars on the road in real time.

\begin{wrapfigure}{r}{.4\textwidth}
    \begin{minipage}{\linewidth}
    \centering
    \includegraphics[width=\linewidth]{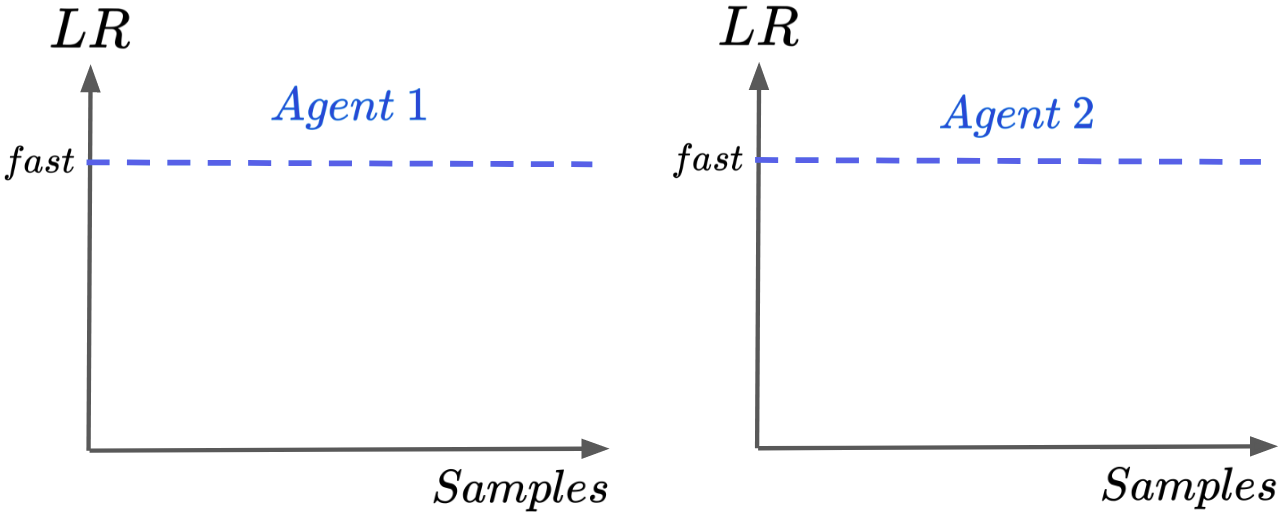}
    \includegraphics[width=\linewidth]{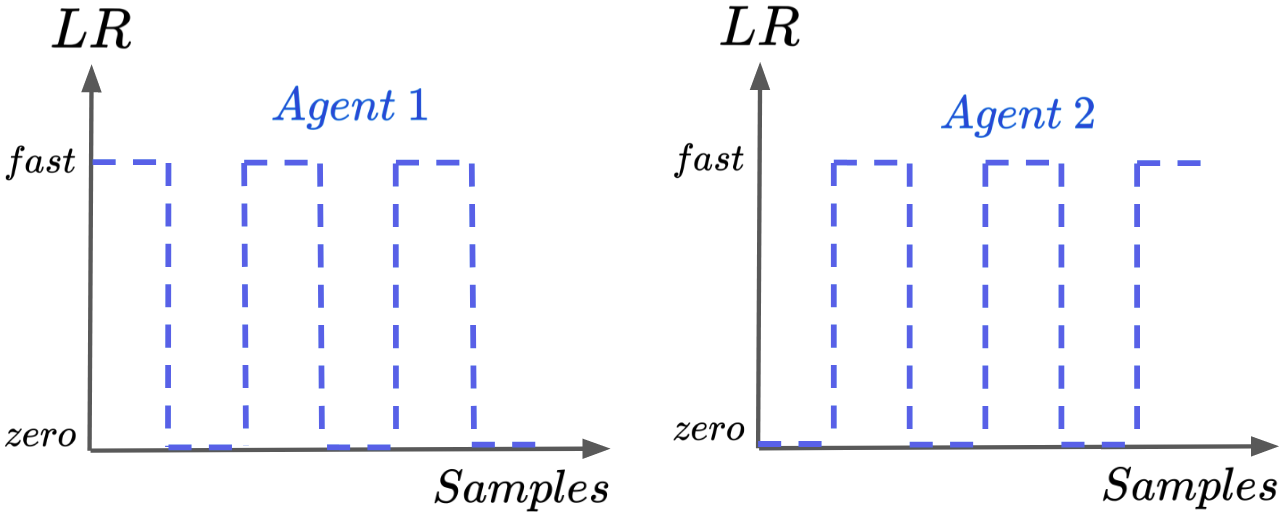}
    \label{fig:independent_sequential_staged_b}
    \includegraphics[width=\linewidth]{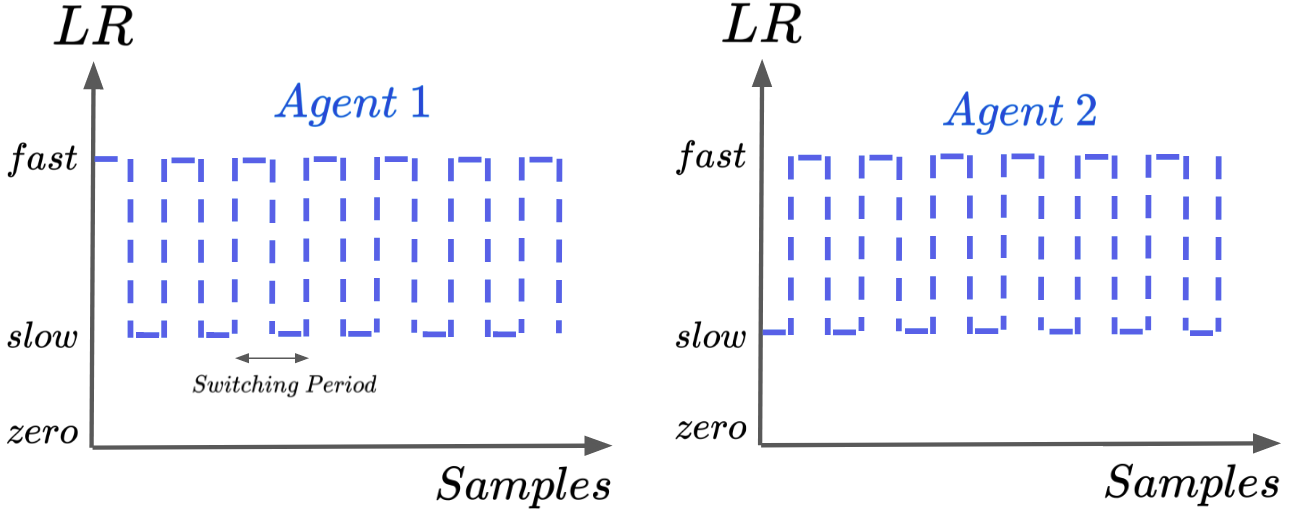}
    \label{fig:independent_sequential_staged_c}
\end{minipage}
\caption{\em Difference in learning rate schedules between (top) independent learning, (middle) sequential learning, and (bottom) an instance of multi-timescale learning.}\label{fig:independent_sequential_staged}
\vskip -2\baselineskip
\end{wrapfigure}


An alternative that does not suffer from the limitations of CTDE is decentralized training, which is the focus of this work. In decentralized training, each agent has access to only its local observations and actions. Here, the challenge of non-stationarity becomes more pronounced. A commonly used scheme for decentralized cooperative deep MARL is to approximate what we call~\emph{independent iterative best response} (IIBR) (a form of independent learning) where agents independently and concurrently try to find the best response strategy with respect to other agents' policies. Examples of applying independent learning to MARL include independent PPO (IPPO)~\citep{de2020independent} and independent Q-learning (IQL)~\citep{tampuu2017multiagent}.

An alternate scheme one could use for decentralized cooperative deep MARL is to approximate what we call \emph{sequential iterative best response} (SIBR) (a form of sequential learning), where instead of all agents learning simultaneously, they learn sequentially one after another. 
The idea of SIBR goes back to fictitious play~\citep{brown1951iterative} and has been used by the game theory community widely, but mostly overlooked by the deep MARL community.
In SIBR, an agent updates its policy until convergence to the Best Response (BR) strategy for other agents' fixed policies. The updated agent's policy is then fixed and the next agent updates its policy to the BR strategy of the other agents' fixed policies and the cycle continues.
This way, only one agent learns at any time, while all other agent's policies remain fixed, circumventing the non-stationarity caused by other agents' policies changing while learning. 
We showcase an example in section \ref{SBR_PBR} where IIBR does not converge while SIBR converges and prove that this is a general property of SIBR that is guaranteed to converge to an agent-by-agent optimal solution (also called team-Nash equilibrium).

Although sequential learning can completely side-step the challenge of non-stationarity introduced by other agents learning concurrently, it slows down the learning process because only one agent is learning at any time. To address this issue, we introduce \emph{multi-timescale learning} framework where all agents learn simultaneously, but at different learning rates. In our proposed method, instead of keeping the ``non-learning'' agents stationary, we allow them to learn using a slower learning rate.
The difference between independent learning, sequential learning, and multi-timescale learning is illustrated in Figure \ref{fig:independent_sequential_staged}.

Our hypothesis is that using multi-timescale learning in the way described above can help the currently used independent learning algorithms to deal with non-stationarity better and thereby improve the performance on cooperative decentralized deep MARL tasks.
In this work we propose Multi-timescale PPO (MTPPO) and Multi-timescale Q-learning (MTQL) as multi-timescale versions of the two commonly used decentralized deep MARL algorithms: Independent PPO (IPPO)~\citep{de2020independent} and Independent Q-learning (IQL)~\citep{tampuu2017multiagent}. We evaluate MTPPO and MTQL on $12$ complex cooperative MARL environments from the \textit{epymarl}~\citep{papoudakis2020benchmarking} testbed. Our results show that multi-timescale algorithms outperform both independent and sequential deep MARL algorithms in most of the tasks.
We perform a detailed analysis to understand the performance of multi-timescale learning.
Multi-timescale learning is a simple idea that has been usually overlooked by deep MARL practitioners, but it can have a significant improvement in performance and therefore should be one of the tools used in decentralized cooperative deep MARL.


\section{IIBR vs SIBR} \label{SBR_PBR}

In this section, we provide an example where IIBR does not converge while SIBR converges and prove that this is a general property of SIBR that is guaranteed to converge to an agent-by-agent optimal solution.

Consider an $n$-agent MARL problem. Let $\theta^i$ denote the policy parameters of agent $i$ and let $\theta=(\theta^1, \cdots, \theta^{n})$ denote the policy parameters of the $n$-agents team. Given policy parameters $\theta$, we use $J(\theta)$ to denote the performance of the team. Furthermore, the notation $\theta^{-i}$ refers to the policy parameters of all agents other than that of agent $i$'s.
We consider iterative methods of the form: $\theta_{t+1} = F_t(\theta_t)$ to update policy parameters, where $\theta_t$ is the policy parameters at iteration $t$ and $F_t$ is some generic update function. Best Response (BR) dynamics is a popular class of update scheme and the most common form of iterative BR dynamics are IIBR and SIBR.

\noindent\textbf{IIBR} is an iterative policy update scheme where at iteration $t$, for all $ i \in \{1,\cdots,n\}$, agent $i$ chooses its policy parameters to be the BR to $\theta^{-i}_{t}$. $$\theta_{t+1}^i=\argmax_{\theta^i} J(\theta^i, \theta^{-i}_{t}) \;.$$
\textbf{SIBR} is an iterative policy update scheme where at iteration $t$, only agent $j = (t \bmod{n})+1$ updates its policy parameters to be the BR to $\theta^{-j}_{t}$, and all other agents remain frozen. 

\[
    \theta_{t+1}^i=
\begin{cases}
    \argmax_{\theta^i} J(\theta^i, \theta_{t}^{-i}),& \text{if } i = (t \bmod{n})+1,\\
    \theta^i_{t}              & \text{otherwise.}
\end{cases}
\]
Note that IIBR suffers from non-stationarity of the environment because all agents are updating in parallel but SIBR does not suffer from that non-stationarity. We now present an example which shows that the non-stationarity can lead to non-convergence of IIBR while SIBR converges.

\begin{example}
Consider a multi-agent estimation problem for minimizing 
team mean-squared
error introduced by~\citet{afshari2021multi}. There are three agents, indexed by $i \in
\{1,2,3\}$, which observe the state of nature $x \sim \mathcal{N}(0,1)$ with
noise. In particular, the observation $y_i \in \mathbb{R}$ of agent~$i$ is
$y_i = x + v_i$, where $v_i \sim \mathcal{N}(0,0.5)$ and $(x,v_1,v_2,v_3)$ are
independent. 
Agent~$i$ generates an estimate $\hat z_i = \mu_i(y_i) \in \mathbb{R}$ based
on its local observations. The objective in the multi-agent estimation problem is to minimize the team mean-squared estimation error $\mathbb{E}\bigl[
    \bigl(x - \frac{1}{3}\sum_{i=1}^3 \hat z_i\bigr)^2 
  \bigr]$.
Minimizing team mean-squared estimation error requires the agents to cooperate to minimize the distance between the average of their estimations and the true state of nature.

As shown in~\citet{afshari2021multi}, the optimal estimation policy is linear, i.e.,
$\hat z_i = K_i y_i$, where the gains $K_i$ are given by the solution of the
following system of linear equations derived by writing the first-order optimality conditions for the total expected error and setting the derivative to zero (for more details refer to appendix \ref{app:MTMSE}).

Iterative best response corresponds to solving the system $\Gamma K =
\eta$ iteratively as $K^{(t+1)} = M^{-1}(N K^{(t)} + \eta)$ for appropriate choice of $M$ and $N$. 
This may be viewed as a linear system $K^{(t+1)} = A K^{(t)} + B \eta$, which is stable when the eigenvalues of $A$ lie within the unit circle. 

We now compute the $A$-matrix for IIBR and SIBR. 
For ease of notation, we will write $\Gamma =
D + L + U$ where $D$ is the diagonal entries, $L$ is the lower triangular
entries (excluding the diagonal) and $U$ is the upper triangular entries (excluding the diagonal).
In IIBR, all agents update their policy at the
same time. So, for this example, IIBR is same as
the Jacobi method for solving a system of linear equations for which $M = D$ and $N = -(L+U)$. Hence $A_{\textit{IIBR}} = - D^{-1}(L+U)$.
Note that the eigenvalues of $A_{\textit{IIBR}}$ are $\{ -\frac 43, \frac
23, \frac 23\}$. Thus, the spectral radius of $A_{\textit{IIBR}}$ is
$\frac 43 > 1$ which is outside of the unit circle. Hence, IIBR does not converge. 

In SIBR, agents update their policies one by one.
So, for this example, the sequential iterative best response is the same as the Gauss
Seidel method for solving a system of linear equations for which $M = (D + L)$
and $N = -U$. Hence, $A_{\textit{SIBR}} = -(D + L)^{-1} U$.
Note that the eigenvalues of $A_{\textit{SIBR}}$ are $\{0, \frac{1}{27}( 14
\pm \sqrt{20} i) \}$. Thus, the spectral radius of $A_{\textit{SIBR}}$ is
$6\sqrt{6}/27 < 1$. Hence, SIBR converges. 
\end{example}

This example illustrates that SIBR which circumvents the problem of non-stationarity converges, while IIBR does not. We now show that this is a general property of SIBR, i.e., SIBR is guaranteed to converge.

\noindent\begin{proposition}
When the per-step rewards are bounded, then the performance of SIBR converges. Moreover, the policy parameters of the agents converge to an agent-by-agent policy. In particular, let $\theta^i_t$ denote the policy parameters of agent~$i$ at time~$t$. Let $\theta^{i\star}$ be any limit point of $\{\theta^i_t\}$ along the time-steps when agent~$i$ updates its policy parameters. Then $(\theta^{1\star}, \theta^{2\star})$ is agent-by-agent optimal.
\end{proposition}

\begin{proof}
For the simplicity of exposition, we consider a $2$ player team game. The same argument applies to a general $n$ player team game.  Let $(\theta^1_{t}, \theta^2_{t})$ denote the parameters of player~$1$ and player~$2$ at iteration~$t$ and let 
 $J(\theta_{t}^1, \theta_{t}^2)$ denote the performance of the team. We assume that players update their policies following SIBR in the order $1 \to 2 \to 1 \to 2 \to \cdots$. At odd iterations 
 $\theta^1_{(2t + 1)} = \argmax_{\theta^1} J(\theta^1, \theta^2_{(2t)}) \; \text{and} \; \theta^2_{(2t+1)} = \theta^2_{(2t)}.
$
Similarly, at even iterations
$\theta^1_{(2t)} = \theta^1_{(2t-1)}
 \; \text{and} \;
   \theta^2_{(2t)} = \argmax_{\theta^2} J(\theta^1_{(2t-1)}, \theta^2).
$

Therefore, we have
  \[
    J\left(\theta_0^1, \theta_0^2\right) \leq J\left(\theta_{1}^1, \theta_{1}^2 = \theta_{0}^2\right) \leq J\left(\theta_{2}^1 = \theta_{1}^1, \theta_{2}^2\right) \leq \cdots. 
\]
For any iteration, we have
\begin{multline*}
    J\left(\theta_{(2t)}^1 = \theta_{(2t-1)}^1, \theta_{(2t)}^2\right) \leq J\left(\theta_{(2t+1)}^1, \theta_{(2t+1)}^2 = \theta_{(2t)}^2\right)
     \leq J\left(\theta_{(2t+2)}^1 = \theta_{(2t+1)}^1, \theta_{(2t+2)}^2\right) \leq \cdots. 
\end{multline*}
$J(\theta^1_{t}, \theta^2_{t})$ is a non-decreasing sequence and is bounded from above (because the rewards are bounded). Hence, it must converge to a limit. Let $J^\star$ denote this limit.
Moreover, let $\theta^{1\star}$ be any limit point of $\{\theta^{1}_{2t+1}\}_{t \ge 1}$ and $\theta^{2\star}$ be any limit point of $\{\theta^2_{2t}\}_{t \ge 1}$. Then, it must be the case that
\[
  \max_{\theta^1} J(\theta^1, \theta^{2\star}) = J^{\star}
  \quad\text{and}\quad
  \max_{\theta^2} J(\theta^{1\star}, \theta^2) = J^{\star}.
\]
Thus, $(\theta^{1\star}, \theta^{2\star})$ is an agent-by-agent optimal solution.
\end{proof}
Note that SIBR is guaranteed to converge only to an agent-by-agent optimal solution, where unilateral deviations by an agent do not improve performance. This is a weaker notion of a solution than global optimality. However, under certain assumptions (such as when $J$ is concave in $(\theta^1, \theta^2)$), agent-by-agent optimality implies global optimality. See~\citet{marshack1972economic, yuksel2013stochastic} for a discussion.

\section{Multi-Timescale Learning for Decentralized Cooperative Deep MARL}\label{mtsl_for_deeprl}

\begin{wrapfigure}{r}{0.35\textwidth}
\vspace*{-8mm}
\includegraphics[width=0.35\columnwidth]{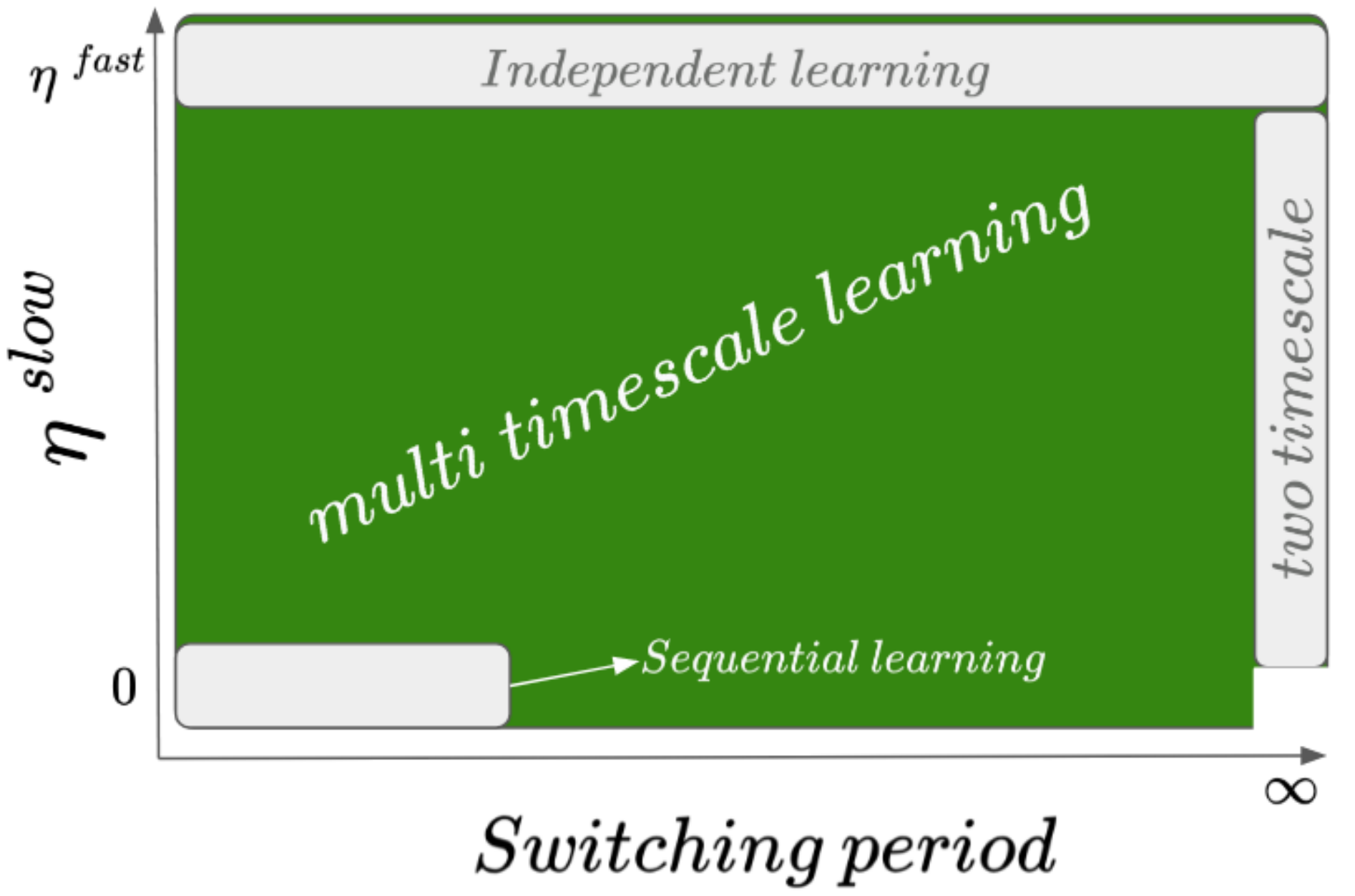}
\caption{\em Multi-timescale learning as a unified framework for independent learning, sequential learning, and two-timescale stochastic approximation. The $x$-axis indicates how fast agents switch between timescales and the $y$-axis indicates the learning rate of the slower agent.}
\label{fig:MTMARL}
 \vspace*{-5mm}
\end{wrapfigure}

In settings of consideration in this paper, where the true model is unknown, and the environment is large and complex that the policy is parameterized using deep neural network function approximator, it is not possible to calculate the exact BR to other agents' policies. 
Agents approximate BR by computing noisy gradients and using gradient ascent iteratively to update their policy parameters. 
In a gradient-based implementation of SIBR, there is only one active agent that performs gradient ascent at any time.
All agents other than the active agent keep their policy parameters frozen which slows down the overall learning process.

Our key insight is that we can speed up overall learning while still minimizing the perceived non-stationarity by allowing the non-active agents to update their policy parameters as well but at a slower timescale. This is an instance of what we call \emph{multi-timescale learning}, where all agents update their policies concurrently, but at different timescales (learning rates). 
Inspiration for using multi-timescale learning comes from two-timescale stochastic approximation methods~\citep{borkar1997stochastic}, which are recursive algorithms in which some of the parameters are updated using smaller step-sizes compared to the remaining parameters~\citep{konda2004convergence}. This principle has been also widely used to train actor-critic algorithms \citep{konda1999actor}. However, we use the idea of multi-timescale algorithms in a different manner, as we explain below.

To illustrate multi-timescale learning, let us consider $n$ agents $\{\theta^i\}_{i=1}^n$ getting trained with $H$ levels of learning rates $\{\eta^h\}_{h=0}^{H-1}$. We can divide the agents to clusters of $\{\vc^h\}_{h=0}^{H-1}$ where $\vc^h$ are the agents trained with learning rate $\eta^h$. 
The switching period ($s$) controls how frequently agents rotate among different clusters (timescales). For example, in the case of 3 agents with $H=2$ and $s=100$, the agents in the clusters $\vc^0$ and $\vc^1$ change as follows: $\vc^0 = \{\theta^0\}$ and $\vc^1 = \{\theta^1, \theta^2\}$ for the first 100 training steps $(t)$, then $\vc^0 = \{\theta^1\}$ and $\vc^1 = \{\theta^0, \theta^2\}$ for $100< t \leq 200$, and $\vc^0 = \{\theta^2\}$ and $\vc^1 = \{\theta^0, \theta^1\}$ for $200< t \leq 300$, and this pattern repeats. All agents in $\vc^0$ will be trained with $\eta^0$, while all the agents in $\vc^1$ will be trained with $\eta^1$.

Even though multi-timescale learning can be implemented with more than two timescales, in this work, we focused on having only two timescales. In any period, one agent $\vc^0 = \{\theta^i\}$ learns at rate $\eta^{\text{fast}}$ and all other agents $\vc^1 = \{\theta^{-i}\}$ learn at rate $\eta^{\text{slow}}$.
If we set $\eta^{\text{slow}} = \eta^{\text{fast}}$, then multi-timescale learning reduces to independent learning. Similarly, if we set $s = \infty$, multi-timescale learning reduces to a standard two-timescale stochastic approximation, where agents are learning at different learning rates, but the learning rates do not switch over time. Furthermore, if we set $\eta^{\text{slow}} = 0$ 
multi-timescale learning reduces to sequential learning, where only one agent at a time updates its gradient. Thus, independent learning, sequential learning, and two-timescale stochastic approximation can all be considered special cases of multi-timescale learning. This is illustrated in Figure~\ref{fig:MTMARL}.

\begin{wrapfigure}{r}{0.5\textwidth}
\vspace*{-10mm}
\begin{minipage}{\linewidth}
\begin{algorithm}[H]
\begin{flushleft}
\textbf{Input} learning rate schedules including $\{\eta^{\text{fast}}, \eta^{\text{slow}}\}$ and switching period $s$ \\
\textbf{Initialize} actors $\pi^i(\theta)$ and critics $Q^i(\phi)$, $i=\{1, \dots, n\}$\\
\textbf{Initialize} faster agent $i^* = 1$ \\
\For{t=1 to \textit{max-train-steps}}{
    \If{ $t \pmod{s} == 0$}{
    set faster agent $i^* = (i^* + 1 \mod{n} ) + 1$
    }
    $\pi^{i^*}(\theta), Q^{i^*}(\phi) \leftarrow \text{PPO update step with} \; \eta^{\text{fast}}$\\
    $\pi^i(\theta), Q^i(\phi) \leftarrow \text{PPO update step with} \; \eta^{\text{slow}}$, $\forall i \neq i^*$
}
\caption{Multi-timescale PPO}
\label{pseudocode:crossval}
\end{flushleft}
\end{algorithm}
\end{minipage}
\vspace*{-10mm}
\end{wrapfigure}

In this paper, we propose multi-timescale versions of two commonly used decentralized cooperative deep MARL algorithms: IPPO and IQL. They are the methods with the best performance among all decentralized MARL algorithms on the various tasks in the \textit{epymarl} benchmark~\citep{papoudakis2020benchmarking}.
We propose Multi-timescale Proximal Policy Optimization (MTPPO) which is based on the IPPO algorithm and Multi-timescale Q-Learning (MTQL) which is based on the IQL algorithm. The pseudo-code for MTPPO  is shown in Algorithm~\ref{pseudocode:crossval}; the pseudo-code for MTQL can be expressed similarly.
\section{Experiments}\label{exp}
We evaluate our hypothesis that agents learning at different timescales improve decentralized cooperative deep MARL compared to agents learning independently at one timescale through rigorous experiments. This section is organized as follows: in section \ref{exp_setup}, we explain the experimental setup used, in section \ref{mtsl_study}, we compare the performance of using multi-timescale learning with independent learning, and in section \ref{analysis}, we provide a detailed analysis of the experimental results using multi-timescale learning.


\subsection{Experimental Setup}\label{exp_setup}
We consider 12 tasks from four different and complex cooperative MARL environments from \textit{epymarl} benchmark \citep{papoudakis2020benchmarking}: Multi-Agent Particle Environment (MPE)~\citep{lowe2017multi},  StarCraft Multi-Agent Challenge (SMAC)~\citep{samvelyan2019starcraft}, Level-Based Foraging (LBF)~\citep{albrecht2015game}, and Multi-Robot Warehouse (RWARE)~\citep{christianos2020shared}.
A brief description of the environments and tasks is provided below. 

\noindent\textbf{Multi-Agent Particle Environment (MPE)~\citep{lowe2017multi}:} MPE environment comprises two-dimensional navigation tasks that require coordination to be solved. We include three tasks from the MPE environment: Speaker-Listener, Adversary, and Tag. 

\noindent\textbf{Level-Based Foraging (LBF)~\citep{albrecht2015game}:}
In the LBF environment, agents should cooperate to collect food items that are scattered
randomly in a grid-world. We include three tasks from LBF environment: 8$\times$8-2p-2f-c, 10$\times$10-3p-3f and 15$\times$15-4p-3f with varying world-size, number of agents and food items. 

\noindent\textbf{Multi-Robot Warehouse (RWARE)~\citep{christianos2020shared}:} RWARE simulates a grid-world warehouse in which agents
(robots) must locate and deliver requested shelves to workstations and return them after delivery. We include three partially observable tasks from RWARE environment: tiny-4ag, tiny-2ag and small-4ag.  The convention for environment name is \{grid-size\}-\{player count\}ag.

\noindent\textbf{StarCraft Multi-Agent Challenge (SMAC)~\citep{samvelyan2019starcraft}:} SMAC simulates battle scenarios in which a team of controlled agents must destroy an enemy team. We include three tasks from SMAC environment: MMM2 (10 agents), 3s5z (8 agents) and 3s\_vs\_5z (3 agents) with a different number of agents and levels of difficulty.
For all these environments, refer to Appendix \ref{app:env-hyp-details} for detailed descriptions of the tasks and hyperparameters.

\begin{figure*}[t!]
\centering
\begin{subfigure}[MTPPO (final performance)]{
    \centering
    \includegraphics[ trim={0 0 0 10cm}, width=0.48\linewidth]{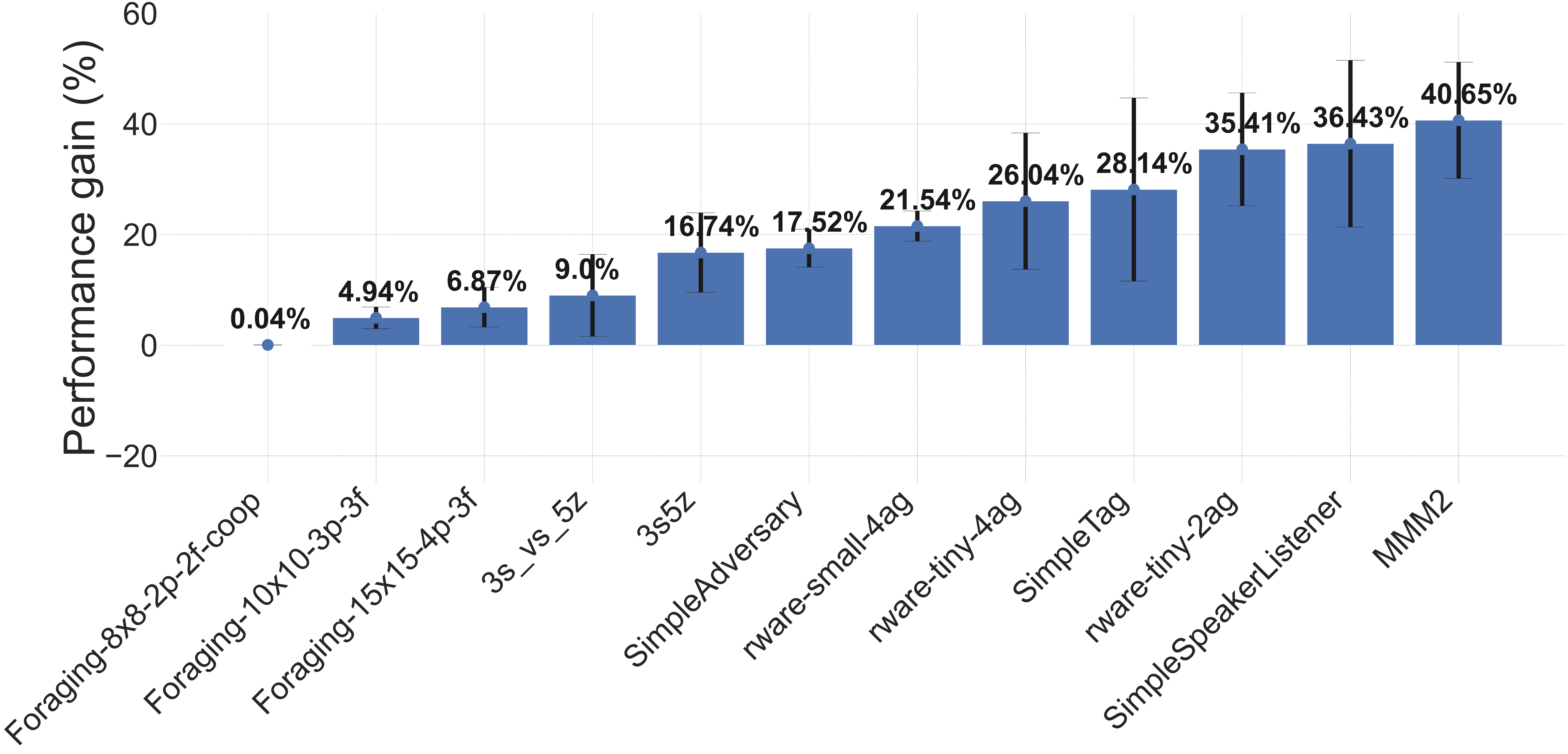}
    }
\end{subfigure}%
\begin{subfigure}[MTQL (final performance)]{
    \centering
    \includegraphics[ trim={0 0 0 0},clip, width=0.48\linewidth]{figs/final_performance/iql_staged_ns_Final_gain.pdf}
    }
\end{subfigure}%
\begin{subfigure}[MTPPO (AUC)]{
    \centering
    \includegraphics[ trim={0 0 0 10cm}, width=0.48\linewidth]{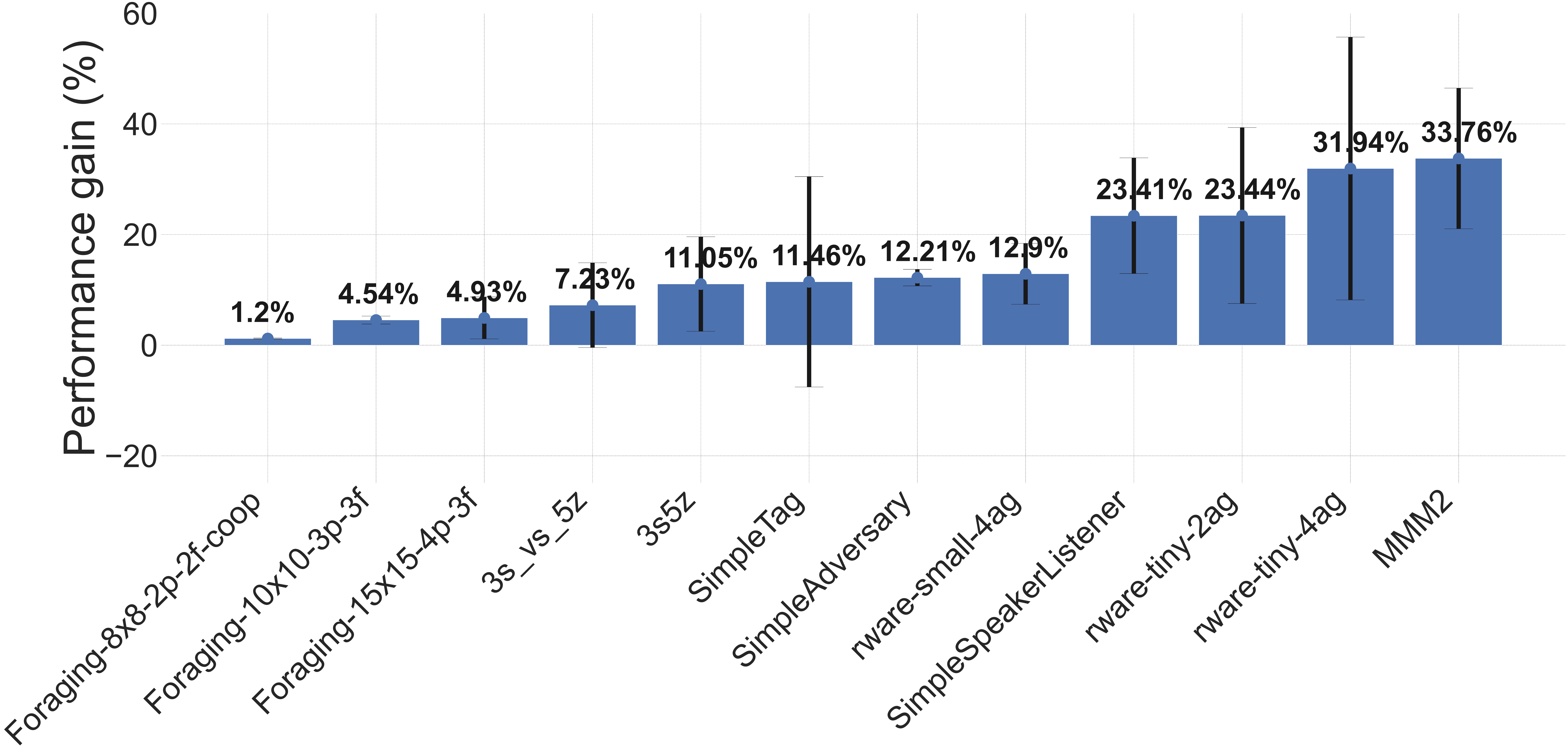}
    }
\end{subfigure}%
\begin{subfigure}[MTQL(AUC)]{
    \centering
    \includegraphics[ trim={0 0 0 0},clip, width=0.48\linewidth]{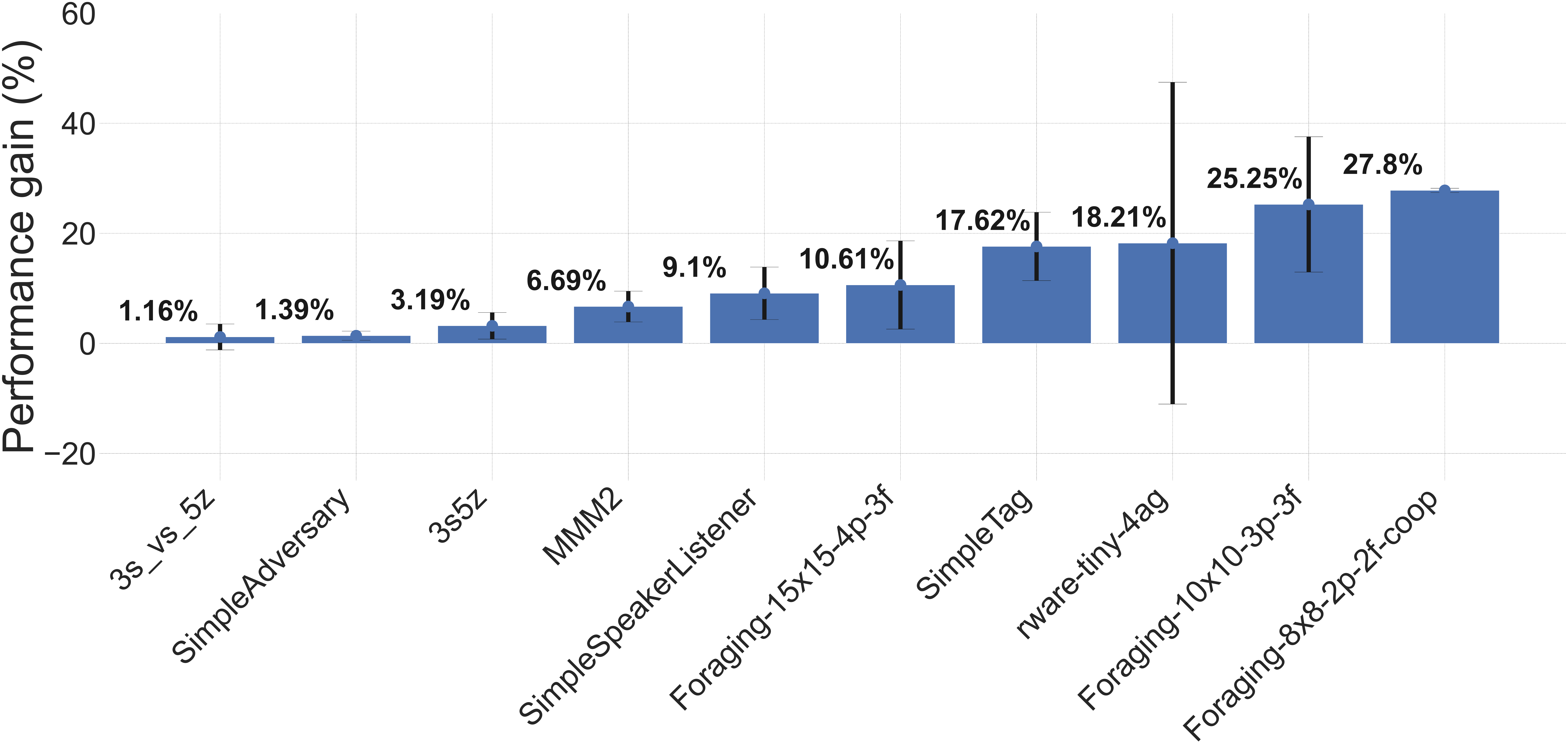}
    }
\end{subfigure}%
\caption{\em Performance gain of MTPPO and MTQL relative to IPPO and IQL on all 12 tasks. IQL and MTQL almost have zero return on small-4ag and tiny-2ag. That's why these two tasks are excluded from (b). MTPPO always either improves or performs as well as IPPO with the highest gains in MMM2, RWARE-tiny-\{4ag, 2ag\}, and MPE-Speaker-Listener. Similar performance gains can be seen with MTQL as well across all tasks with maximum gains in RWARE-tiny-4ag and Foraging-10$\times$10-3p-3f. Error bars represent the standard error over 5 seeds.}
\label{fig:perf-gain} 
\vspace{-5mm}

\end{figure*}



\subsection{Multi-Timescale vs Independent Learning}\label{mtsl_study}

In all our multi-timescale experiments, we have two timescales with learning rates $lr_0$ and $lr_1$. In the case of two agents, each agent learns with the respective learning rate, while in the case of more than two agents, one agent learns with $lr_0$ and the rest with $lr_1$. 
We perform hyperparameter search by varying the learning rates $(lr_0, lr_1)$ over $L \times L$, where $L$ is chosen by considering the learning rates around the best hyperparameters reported in~\citet{papoudakis2020benchmarking}. Refer to Appendix~\ref{app:env-hyp-details} for details. We vary the switching period hyperparameter $s \in \{1, 10, 10^2, 10^3, 10^4\}$. In IPPO, critic and actor updates might be done with different frequencies. In such a case switching is done based on critic training steps.
%

We compare the performance of MTPPO and MTQL with those of IPPO and IQL. Note that, by definition, multi-timescale learning includes $lr_0 = lr_1$ as well. Therefore, to ensure that any improvement in the performance reported is the result of having \textit{different} learning rates, we do not report the results of $lr_0 = lr_1$ for MTPPO and MTQL.

\begin{wraptable}{r}{4cm}
    \vskip -\baselineskip
    \centering
    \caption{Aggregate performance across all 12 tasks.}
    \begin{tabular}{@{}ccc@{}}
    \toprule
        &  Mean & Median\\
    \midrule
        IQL & $0.363$ & $0.464$  \\
         MTQL & $0.404$ & $0.507$  \\
    \midrule
        IPPO & $0.534$ & $0.578$  \\
         MTPPO & $0.599$ & $0.714$  \\
    \bottomrule
    \end{tabular}
    \vskip -\baselineskip
    \label{tab:aggregate_peformance}
\end{wraptable}
Following~\citet{papoudakis2020benchmarking}, we normalize the returns of all algorithms in each task in the [0, 1] range using the following formula: $(G^a_t - min(G_t)) / (max(G_t) - min(G_t))$ where $G^a_t$ is the return of algorithm $a$ in task $t$, and $G_t$ is the returns of all algorithms in task $t$. As shown in Table~\ref{tab:aggregate_peformance}, we report the aggregate performance of the algorithms across all the 12 environments.

Figure~\ref{fig:perf-gain} shows the performance gain of MTPPO relative to IPPO as well as MTQL relative to IQL across the 12 tasks. MTPPO always either improves or performs as good as IPPO with highest gains in MMM2, RWARE-tiny-\{4ag, 2ag\}, and MPE-Speaker-Listener. Similar performance gains can be seen with MTQL as well across all tasks with maximum gains in RWARE-tiny-4ag and Foraging-10$\times$10-3p-3f. This also shows that multi-timescale learning can be effective for tasks with different varying numbers of agents (from 2 - 10 in our tasks). 
Detailed observations on each environment are as follows:  

\noindent\textbf{MPE:} In the case of Speaker-Listener and Adversary, MTPPO almost always performs better than IPPO while in the case of Tag, MTPPO and IPPO have comparable performance. MTQL clearly performs better than IQL in all tasks.

\noindent\textbf{LBF:} MTPPO performs better than IPPO in all these environments.
A similar trend is seen with MTQL and IQL. 

\noindent\textbf{RWARE:} As we can observe from Figure~\ref{fig:perf-gain}, MTPPO almost always performs better than IPPO in the case of tiny-4ag and tiny-2ag, while the performance is comparable in small-4ag. IQL and MTQL almost have zero return on small-4ag and tiny-2ag. To avoid reporting misleading performance gains, these two tasks are excluded. However, MTQL clearly performs better than IQL on tiny-4ag. 

\noindent\textbf{SMAC:} In all three tasks, MTPPO performs consistently better than IPPO. MTQL outperforms IQL in MMM2 and 3s5z, while we see comparable performance in 3s\_vs\_5z.

\subsection{Analysis}\label{analysis}
We perform a detailed analysis of our experiments to study the following questions:


\begin{figure}[t!]
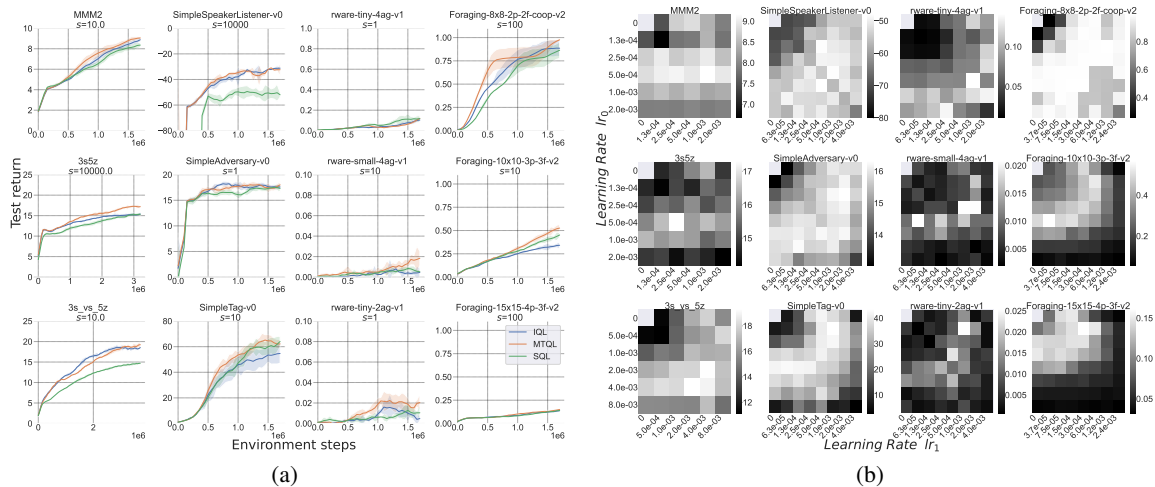

\centering
\begin{subfigure}[]{
    \centering
    \includegraphics[ width=0.45\linewidth]{figs/test_curves/iql_staged_ns_curves.pdf}
    }
\end{subfigure}
\begin{subfigure}[]{
    \centering
    \includegraphics[ width=0.45\linewidth]{figs/heatmaps/collas_iql_staged_ns_best_SIQL_heatmaps.pdf}
    }
\end{subfigure}\label{fig:best-MTQL-heatmaps-learning-curves}
\caption{\em (a) Learning curves for IQL, MTQL, and SQL. MTQL improves the performance in almost all the tasks compared with Sequential Q-learning (SQL) (where one of the learning rates is zero). SQL sometimes outperforms IQL, for example, in SimpleTag and Foraging $10\times10$ while still worse than MTQL.
(b) Final performance of IQL (diagonal), SQL (top row and leftmost column), and MTQL (the entire grid) with different learning rate combinations. These heatmaps are for the best switching period values. In many tasks, the best performance results from an off-diagonal learning rate combination, which is possible only through multi-timescale learning.}
\label{fig:best-MTQL-heatmaps-learning-curves} 
\vspace{-5mm}
\end{figure}

\medskip\noindent
\textbf{Does multi-timescale learning accelerate sequential learning?} Figure~\ref{fig:best-MTQL-heatmaps-learning-curves}(a) shows the learning curves for the best version of IQL, MTQL, and Sequential Q-learning (SQL). We implement sequential learning by choosing a zero learning rate for the slower agent. The switching period is still optimized for sequential learning. MTQL improves both the performance and sample complexity in almost all the tasks compared with SQL, and also IQL. Interestingly, sequential learning sometimes outperforms independent learning, for example, in SimpleTag and Foraging 10$\times$10-3p-3f while still worse than multi-timescale learning.
For more learning curves, see Appendix~\ref{app:learning_curves_heatmaps}. 



\medskip\noindent
\textbf{How does multi-timescale learning's performance vary with different timescales (learning rates)?} We report the performance results of MTQL, SQL, and IQL for each combination of learning rates across all switching periods in Figure \ref{fig:best-MTQL-heatmaps-learning-curves}. Diagonal values with the same learning rates denote IQL. The top row and the leftmost column of the heatmaps show the performance of SQL, and the remaining  off-diagonal values represent multi-timescale learning excluding independent learning and sequential learning. In several tasks, there seems to be a pattern with higher performance resulting from these off-diagonal learning rates corresponding to multi-timescale learning rates which are non-zero and different from each other. For the heatmaps of MTPPO, refer to Appendix~\ref{app:learning_curves_heatmaps}.

\begin{figure}[t!]
\vskip -\baselineskip
\centering
    \includegraphics[ width=0.95\linewidth]{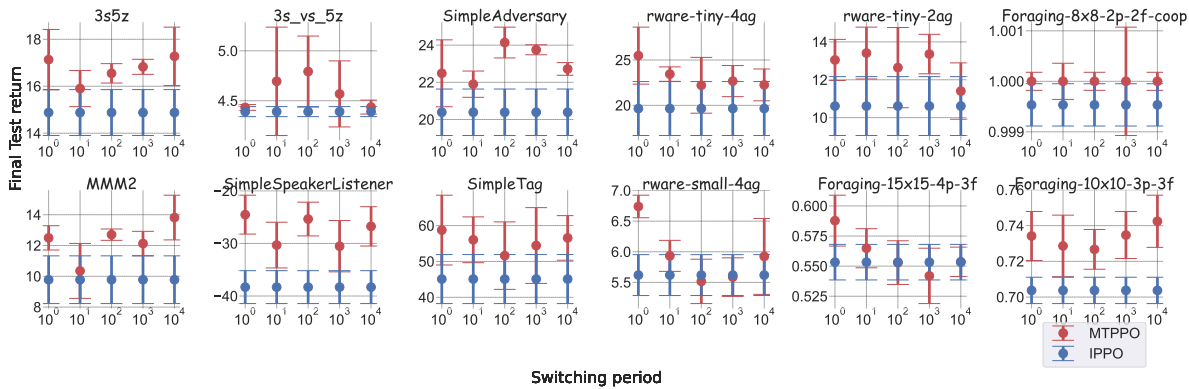}
\caption{\em Performance of IPPO, IQL, and their Multi-timescale version vs switching period for each task. At each switching period, the best-performing multi-timescale learning with different learning rates is compared with the best-performing independent learning with the same learning rates. Error bars represent the standard error over 5 seeds.}\label{fig:perf-gain-errorbars}
\vspace{-5mm}
\end{figure}

\medskip\noindent
\textbf{How does multi-timescale learning's performance vary with different switching periods?}
Figure~\ref{fig:perf-gain-errorbars} shows MTQL's performances across different switching periods (Refer to Figure~\ref{app:perf-gain-errorbars} for MTPPO's results). Overall multi-timescale learning seems to be robust with respect to the choice of the switching period.
There are two cases where finding a suitable switching period seems necessary for good performance. 
Firstly, in tasks where the difference between the best combination of learning rates is high, such as Foraging-10$\times$10-3p-3f, and Foraging-15$\times$15-4p-3f (Figure \ref{fig:best-MTQL-heatmaps-learning-curves}), we observe in Figure \ref{fig:perf-gain-errorbars} that there is a sweet spot for switching.


\begin{wraptable}{r}{8cm}
    \vskip -\baselineskip
    \centering
    \caption{Gap between CTDE and DT recovered by Multi-timescale learning. For MTPPO, the considered gap is between IPPO and $max$(MAPPO, MAAC2) and for MTQL the gap is between IQL and $max$(VDN, QMIX).}
    \begin{tabular}{@{}ccccc@{}}
    \toprule
        &  MPE & SMAC & LBF & RWARE \\
    \midrule
        MTPPO & $15.63\%$ & $57.65\%$ & $29.82\%$ & $19.09\%$  \\
         MTQL & $56.75\%$ & $22.60\%$ & $18.26\%$ & -- \\
    \bottomrule
    \end{tabular}
    \vskip -\baselineskip
    \label{tab:gap_ctde_dt}
\end{wraptable}

\medskip\noindent
\textbf{On which tasks does multi-timescale learning help the most?} We observe that multi-timescale learning helps more in environments where there is a gap between CTDE and independent learning. 
For example, environments such as MPE Speaker-Listener, hard SMAC, and RWARE, where CTDE outperform independent learning (as shown in~\citet{papoudakis2020benchmarking}), are also environments where multi-timescale learning helps a lot compared to their independent learning counterparts. 
We hypothesize that while CTDE helps reduce non-stationarity by sharing information about other agents observations and actions, thereby improving coordination and reducing variance \citep{lowe2017multi,yang2018cm3, das2019tarmac, li2019robust},
multi-timescale learning improves performance by also reducing non-stationarity, but by controlling non-stationarity that arise from other agents learning concurrently. We would like to emphasize that we do not expect our method to recover all the gap between Decentralized Training (DT) and CTDE since CTDE has access to more information during training than our method that uses DT. However, to provide a picture of where the performance of different methods across different training schemes stand, we computed the percentage of the performance gap between DT and the best CTDE method among Multi-agent PPO (MAPPO)\citep{yu2021surprising}, Multi-agent Actor-Critic (MAAC2)\citep{papoudakis2020benchmarking}, Value Decomposition Networks (VDN) \citep{sunehag2017value}, and QMIX \citep{rashid2018qmix} for each environment, that our proposed Multi-timescale DT bridges as $\frac{(\mathit{MDT}-\mathit{DT})*100}{\mathit{CTDE} - \mathit{DT}}$. As shown in Table~\ref{tab:gap_ctde_dt}, MTPPO and MTQL managed to recover some part of the gap between CTDE and DT.

\begin{figure*}[b!]
\centering
\begin{subfigure}[MTPPO]{
    \centering
    \includegraphics[ trim={0 0 0 10cm}, width=0.48\linewidth]{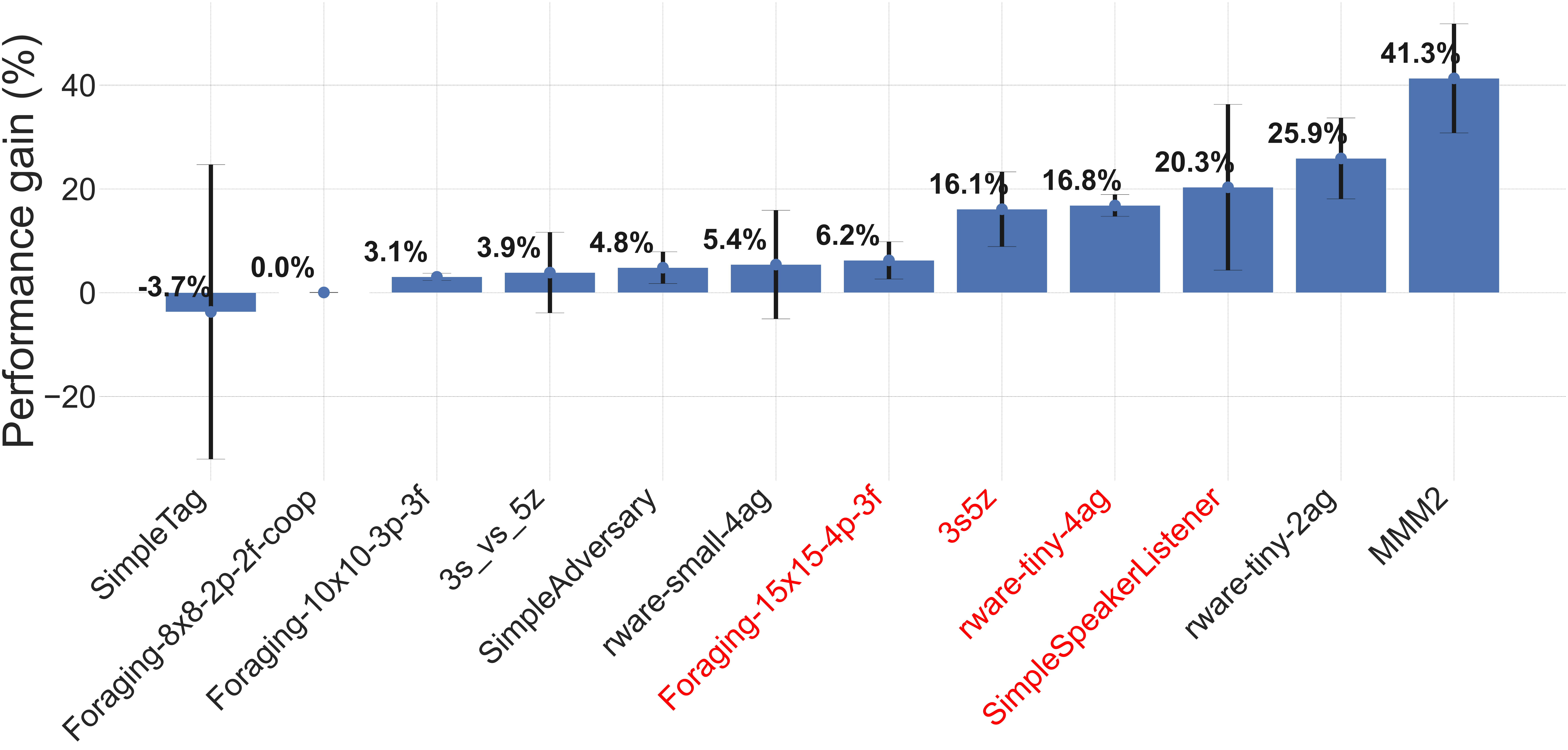}
    }
\end{subfigure}%
\begin{subfigure}[MTQL]{
    \centering
    \includegraphics[ trim={0 0 0 0},clip, width=0.48\linewidth]{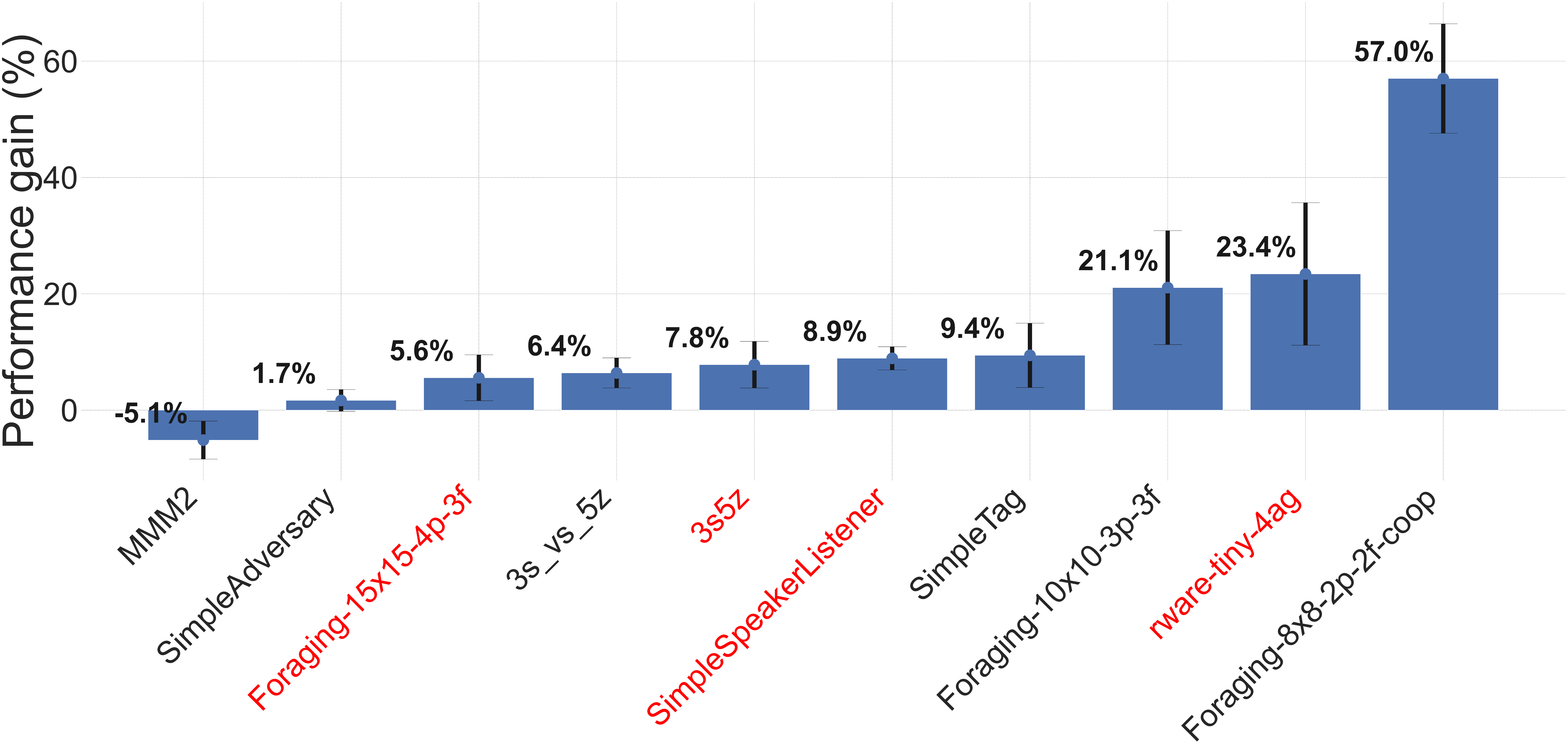}
    }
\end{subfigure}%
\caption{\em Performance gain of MTPPO and MTQL relative to IPPO and IQL on all 12 tasks. \textbf{HPs are tunned only for the tasks shown in red.} These results show that the improvements of MTPPO and MTQL over the baselines are relatively robust to HP tunning. Error bars represent the standard error over 5 seeds. }
\label{fig:perf-gain-subset-HP-tuning} 
\vspace{-5mm}

\end{figure*}
\medskip\noindent
\textbf{How robust multi-timescale learning is with respect to the hyper-parameteres?}
To assess the robustness of multi-timescale learning in relation to hyperparameters (HPs), we employed the HP tuning approach described by~\citet{papoudakis2020benchmarking}. Rather than tuning the HPs individually for each task, we conducted HP tuning for a selected subset of tasks. Specifically, we determined the optimal learning rates for both multi-timescale learning and independent learning on Speaker-Listener, Foraging-15$\times$15-4p-3f, RWARE-tiny-4ag, and 3s5z. Subsequently, we evaluated the performance of these methods on other tasks within the same environment, utilizing the identified hyperparameters. Figure~\ref{fig:perf-gain-subset-HP-tuning} presents the corresponding results. The outcomes are varied: while some tasks showed a decline in relative performance gain, others exhibited a higher performance gain for multi-timescale learning. Overall, our approach consistently outperforms the baseline in the majority of tasks, demonstrating the robustness of multi-timescale learning.

\begin{wrapfigure}{r}
{0.4\textwidth}
\centering
\begin{subfigure}[MTQL performance gain]{
    \includegraphics[ width=1.05\linewidth]{figs/orion_results/SIQL_Orion_final_perf_gain.pdf}
    }
\end{subfigure} \\
\begin{subfigure}[Optimization curve]{
    \includegraphics[ width=1.1\linewidth]{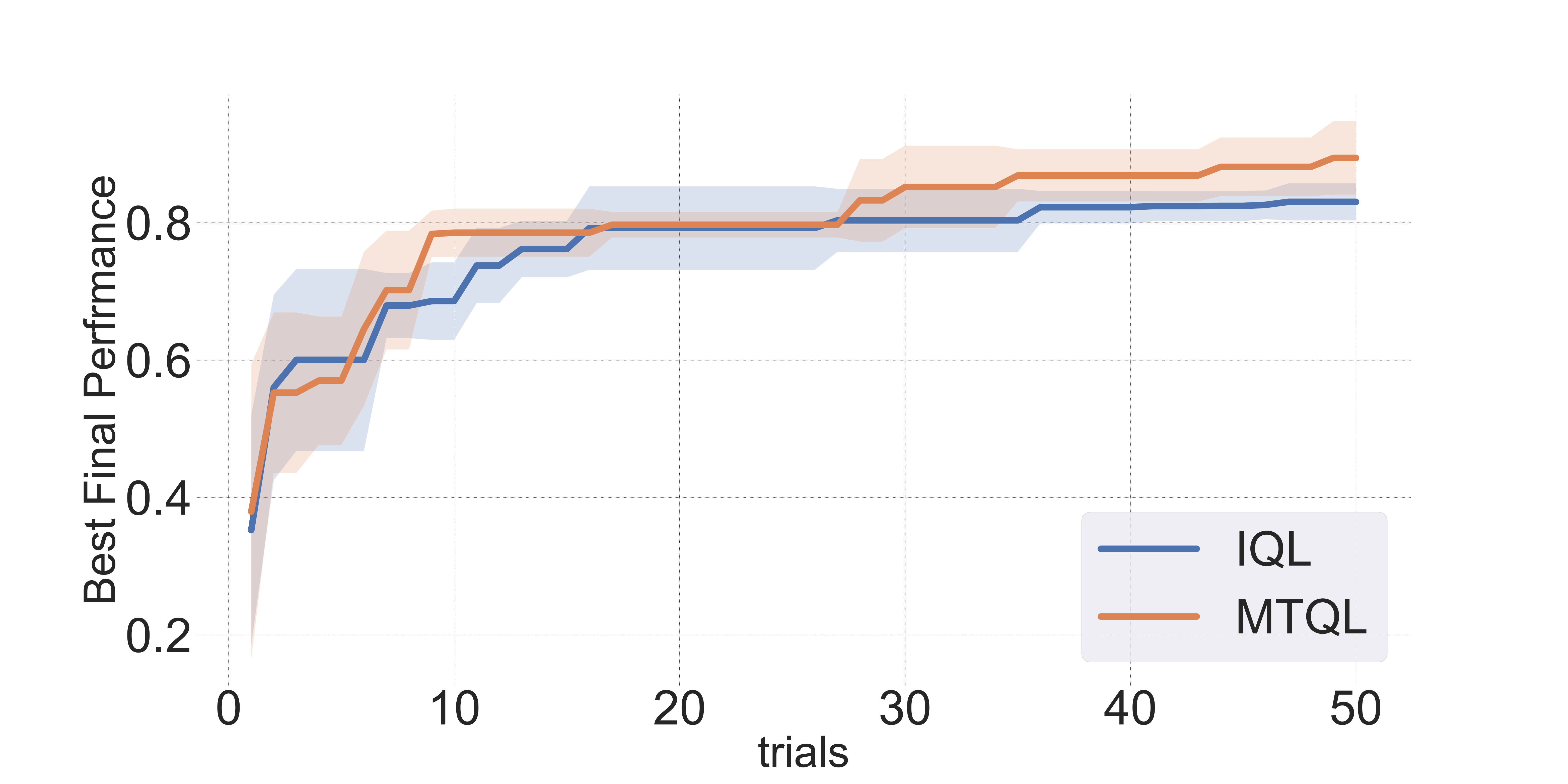}
    }
\end{subfigure}%
\caption{\em (a) Performance gain of MTQL relative to IQL for a budget of 50 trials. In each run, both methods start with the same initial of 20 points, and the TPE method is used for optimization. The error bars represent the standard error across 5 optimization processes. (b)  optimization curve for Foraging-8x8-2p-2f task }
\label{fig:orion_results}
\end{wrapfigure}

\medskip\noindent
\textbf{How does multi-timescale learning perform compared to independent learning under low computation budget?} Here, we test how multi-timescale learning performs relative to independent learning when having access to a fixed low computation budget. To do so, we use hyperparameter tuning tools and show that multi-timescale learning has competitive performance even under a low computation budget. To this end, we perform hyperparameter tuning for both the baselines (only one hyperparameter for IQL and IPPO: $lr0$) and their multi-timescale versions (with three hyperparameters: $lr0$, $lr1$, $s$) with  Orion~\citep{xavier_bouthillier_2022_0_2_3}. In particular, we used the Tree-structured Parzen Estimator (TPE) algorithm which is one of the Sequential Model-Based Global Optimization (SMBO) algorithms. In these experiments, we start with 20 random initial samples of hyperparameters. For each trial, learning rates are sampled from a \textit{loguniform} distribution and the switching period is chosen from a categorical distribution. The final evaluation return is averaged over 3 independent runs and returned to the TPE algorithm to propose the next set of hyperparameters. TPE proposes 30 extra hyperparameters sequentially so the total number of trials is 50 per method. Then we let both the methods run this optimization process 5 independent times and take the average over these runs to get the final curves like in Fig][(b). We report the results of MTQL on three different environments and 10 tasks in Fig~\ref{fig:orion_results}(a). Since the IQL agent almost achieves zero performance on two out of three RWARE tasks, we did not include the results since any amount of improvement had a high variance. See Appendix~\ref{app:orion_exp} for other optimization curves and more results. With the limited budget of 50 trials, MTQL performs better than IQL in 7 out of 10 tasks. 
Note that multi-timescale learning scales better with more compute. As more compute budget is provided, the results will get closer to the result reported in Figure \ref{fig:perf-gain} (b), with significantly more performance improvement on all 10 tasks, and also independent learning is a special case of multi-timescale learning.



\section{Related Work} 
\label{related_work}

%


\textbf{Handling non-stationarity in cooperative MARL:} A lot of work in the cooperative MARL literature has focused on \emph{centralized training and decentralized execution} (CTDE)
\citep{sunehag2017value, lowe2017multi, rashid2018qmix,hostallero2019learning, mao2020information}.
Although CTDE is able to circumvent the conceptual challenges of non-stationarity of the environment, it may not always be possible to perform centralized training in the first place.
For example, in online real-world settings like self-driving cars. Moreover, centralized critic suffers from the curse of agents \citep{mao2022improving, wang2020breaking} --- the size of joint action space increases exponentially with the number of agents.

However, there have been relatively few efforts in handling non-stationarity in decentralized cooperative MARL settings. 
\citet{foerster2017stabilising} propose a method to reduce the effect of non-stationarity in IQL by conditioning each agent's value function on a time-dependent fingerprint and report promising results on StarCraft unit micro-management \citep{samvelyan2019starcraft}.
A theoretical justification for the decentralized IPPO is provided in~\citet{sun2022monotonic}, which guarantees monotonic improvements by forming a trust region over joint policies and provides some insight on how to form the trust region for each agent individually. However, these guarantees do not extend to IQL because it is not based on trust regions.
Self-play for zero-sum games deals with the non-stationarity problem by playing against several past versions of itself. However, in the case of cooperative environments, self-play does not exploit the fact that the agents can cooperate to optimize the common objective.
Self-play also has been shown to learn arbitrary policies that do not generalize when cooperating with novel partners \citep{bard2020hanabi, nekoei2021continuous}.

With respect to agents learning at different rates, there have been some works that are based on optimistic heuristics for updating the learning rates in cooperative environments. Work by~\citet{matignon2007hysteretic}  proposes Hysteretic Q-learning in which the Q-values are updated with a higher learning rate when getting a reward better than the expected state-action value and \citet{omidshafiei2017deep} implemented Deep Hysteretic Q-learning. 
Note that this approach is complementary to multi-timescale learning and will be an interesting future work to evaluate multi-timescale hysteretic Q-learning.
A classic work by~\citet{bowling2002multiagent} focuses specifically on varying the learning rate on a restricted class of iterated matrix games.

\medskip
\noindent\textbf{Sequential learning and two-timescale learning:} The idea of sequential learning goes back to fictitious play~\citep{brown1951iterative}. However, the combination of Fictitious Self-Play (FSP) with deep RL was proposed recently by \citet{heinrich2015fictitious}. \citet{lanctot2017unified} proposed Policy-Space Response Oracles (PSRO) to generalize the IIBR, SIBR, and fictitious play methods. PSRO is focused on tackling overfitting in MARL while our approach aims to tackle non-stationarity in the Decentralized Training setting. Moreover, PSRO does not explore this idea of soft learning where we allow a fast learner and a slow learner simultaneously instead of IIBR (both fast) and SIBR (one fast and one not learning). Finally, in contrast to PSRO, our method is conceptually much simpler since it does not require the creation of a population of BR agents and the computation of meta-strategies.

Recently, sequential learning has been studied by a series of works \citep{bertsekas2020multiagent, bertsekas2021multiagent}, laying down the theoretical foundations for agent-by-agent policy iteration, value iteration methods, and their optimality guarantees. These works show the promise of SIBR but they are still limited to fully-observable settings. In our work, we are proposing a new setting with switching learning rates and we believe more theoretical work on switching Ordinary Differential Equations (ODEs) is needed, which are beyond the scope of this work, but certainly should be done in future works to understand the method that shows very promising empirical results.

There has been some recent work on competitive decentralized training using two-timescale optimization providing convergence guarantees. A two-timescale decentralized algorithm was developed for zero-sum games by~\citet{sayin2021decentralized}, where each agent updates its local Q-function and state-value function estimates concurrently, the latter happening at a slower timescale without requiring asymmetric update rules. Also, \citet{daskalakis2020independent} show that in a zero-sum game, when two competitive policy gradient-based agents learn simultaneously and their learning rates follow a two-timescale rule, their policies converge to a min-max equilibrium. However, these results are all still limited to zero-sum games.

\section{Conclusion and Future Work}\label{conclusions}

The commonly used training scheme for decentralized cooperative deep MARL has been independent learning based on IIBR, which suffers from the non-stationarity of other simultaneous learning agents. Sequential learning on the other hand can circumvent this issue, but it is slow since only one agent learns at any time.
In this work, we proposed using the framework of multi-timescale learning, where different agents are learning concurrently at different learning rates for decentralized cooperative deep MARL. In our proposed instantiation of multi-timescale learning, agents learn one after another like in sequential learning, but while one agent learns, all other agents also concurrently update their policies but at a slower learning rate, minimizing the issue of non-stationarity, while not making the overall learning very slow. 
Our evaluation of  Multi-timescale PPO (MTPPO) and Multi-timescale QL (MTQL) on 12 complex cooperative MARL tasks from the \textit{epymarl} benchmark showed that multi-timescale versions outperform both their independent and sequential counterparts in most of the tasks.

This work empirically presented multi-timescale learning as a promising framework for decentralized cooperative deep MARL. Conducting more theoretical work to understand the learning dynamics in multi-timescale learning can be an exciting future work. Even though in this paper, we focused on two standard decentralized algorithms, multi-timescale learning can be applied to other decentralized and even centralized methods like multi-agent PPO (MAPPO)~\citep{yu2021surprising}. In this work, we evaluated MTPPO and MTQL with only two timescales. We assumed that in the case of more than two agents, only one agent is learning at a different timescale while other agents are learning at the same timescale. Clearly, there are more ways to cluster the agents, which might be useful especially if the environment and task are such that there are dependencies and independencies between a certain subset of agents. Evaluation of multi-timescale learning with other algorithms, more timescales, different clustering protocols, and in non-cooperative settings are very interesting future work. Moreover, in the current setup, the agents need to agree upon the learning rate schedule in advance (what learning rates to use and with what frequencies to switch). Although it is a reasonable assumption that agents can agree to follow some protocol in advance in many MARL scenarios, one potentially promising idea is to adaptively tune the learning rates.
Overall, we hope this work is a first step towards more decentralized cooperative deep MARL methods based on multi-timescale learning. 


\section*{Acknowledgements}

We like to acknowledge Compute Canada for providing compute resources for this work. The work of AS and AM was supported in part by NSERC Alliance Grant. SC is supported by a Canada CIFAR AI Chair and an NSERC Discovery Grant.
JR is supported by IVADO postdoctoral research funding. We also would like to thank Xutong Zhao, Sai Krishna, and Sriyash Poddar for their valuable comments on the paper. 

\bibliography{collas2023_conference}
\bibliographystyle{collas2023_conference}

\newpage
\appendix

\section*{\Large Appendices}

\section{Multi-agent estimation problem}\label{app:MTMSE}
Consider a general three player minimum team mean squared optimization problem with $x \sim \mathcal{N}(0,1)$, $y_i = x + w_i$ where $w_i \sim \mathcal{N}(0, \sigma^2)$ where the objective is to choose $\hat z_i = \mu_i(y_i)$ to minimize an estimation cost of the form
\[
    \mathbf{E}[ (x \mathbf{1} - \hat z)^T S (x \mathbf{1} - \hat z) ]
\]
where $\hat z = \text{vec}(\hat z_1, \hat z_2, \hat z_3)$ and 
\[
S = 
  \begin{bmatrix}
    p & q & q \\[3pt]
    q & p & q \\[3pt]
    q & q & p
  \end{bmatrix} 
\]
According to \cite[Theorem 1]{afshari2021multi}, the team optimal estimation strategies are linear and of the form $\hat z_i = K_i y_i$, where $K = \text{vec}(K_1, K_2, K_3)$ is given by the solution of the linear system of equations
\[
   \Gamma K = \eta
\]
where 
\begin{equation}\label{eq:Gamma}
\Gamma = \begin{bmatrix}
    p(1+\sigma^2) & q & q \\[3pt]
    q & p(1+\sigma^2) & q \\[3pt]
    q &  q & p(1+\sigma^2)
  \end{bmatrix} 
\quad\text{and}\quad 
\eta = \begin{bmatrix}
    p + 2q \\[3pt]
    p + 2q \\[3pt]
    p + 2q
  \end{bmatrix}
\end{equation}

Now, the cost function in the example described in Sec.~\ref{SBR_PBR}, the estimation cost may be written as
\[
    \frac 19 \mathbf{E} \left[
    \begin{bmatrix}
    x - \hat{z}_1 \\[3pt]
    x - \hat{z}_2 \\[3pt]
    x - \hat{z}_3
  \end{bmatrix}^T
  \begin{bmatrix}
    p & q & q \\[3pt]
    q & p & q \\[3pt]
    q & q & p
  \end{bmatrix} \begin{bmatrix}
    x - \hat{z}_1 \\[3pt]
    x - \hat{z}_2 \\[3pt]
    x - \hat{z}_3
  \end{bmatrix} \right]
\]
where $p = 1$ and $q = 1$. In the model, it is also assumed that $\sigma^2 = 0.5$. Thus, \eqref{eq:Gamma} simplifies to:
\[
\Gamma = \begin{bmatrix}
    \frac 32 & 1 & 1 \\[3pt]
    1 & \frac 32 & 1 \\[3pt]
    1 & 1 & \frac 32
  \end{bmatrix} 
\quad\text{and}\quad 
\eta = \begin{bmatrix}
    3 \\[3pt]
    3 \\[3pt]
    3
  \end{bmatrix}.
\]

Iterative best response corresponds to solving the system $\Gamma K =
\eta$ iteratively as $K^{(t+1)} = M^{-1}(N K^{(t)} + \eta)$ for appropriate choice of $M$ and $N$. 
This may be viewed as a linear system $K^{(t+1)} = A K^{(t)} + B \eta$, which is stable when the eigenvalues of $A$ lie within the unit circle. 

We now compute the $A$-matrix for IIBR and SIBR. 
For ease of notation, we will write $\Gamma =
D + L + U$ where $D$ is the diagonal entries, $L$ is the lower triangular
entries (excluding the diagonal) and $U$ is the upper triangular entries (excluding the diagonal).
In IIBR, all agents update their policy at the
same time. So, for this example, IIBR is same as
the Jacobi method for solving a system of linear equations for which $M = D$ and $N = -(L+U)$. Hence $A_{\textit{IIBR}} = - D^{-1}(L+U)$.
\[
  A_{\textit{IIBR}} \coloneqq - D^{-1}(L+U) = 
  \begin{bmatrix}
    0 & -\frac 23 & -\frac 23 \\[3pt]
    -\frac 23 & 0 & -\frac 23 \\[3pt]
    -\frac 23 & -\frac 23 & 0
  \end{bmatrix}.
\]
Note that the eigenvalues of $A_{\textit{IIBR}}$ are $\{ -\frac 43, \frac
23, \frac 23\}$. Thus, the spectral radius of $A_{\textit{IIBR}}$ is
$\frac 43 > 1$ which is outside of the unit circle. Hence, IIBR does not converge. 

In SIBR, agents update their policies one by one.
So, for this example, the sequential iterative best response is the same as the Gauss
Seidel method for solving a system of linear equations for which $M = (D + L)$
and $N = -U$. Hence, $A_{\textit{SIBR}} = -(D + L)^{-1} U$.
\[
  A_{\textit{SIBR}} \coloneqq -(D + L)^{-1} U =
  \begin{bmatrix}
    0 & -\frac 23 & -\frac 23 \\[3pt]
    0 & \frac {4}{9} & -\frac 2{9} \\[3pt]
    0 & \frac {4}{27} & \frac {16}{27} 
  \end{bmatrix}.
\]
Note that the eigenvalues of $A_{\textit{SIBR}}$ are $\{0, \frac{1}{27}( 14
\pm \sqrt{20} i) \}$. Thus, the spectral radius of $A_{\textit{SIBR}}$ is
$6\sqrt{6}/27 < 1$. Hence, SIBR converges.

\section{Environment details and Hyperparameters} \label{app:env-hyp-details}
In this section, we give an overview of the tasks and environments we used for our experiments. Then, we list all the important hyper-parameters.

\textbf{MPE}: These are two-dimensional navigation tasks that require coordination. The observations of the agent include high-level feature vectors like relative agent and landmark locations.

\textbf{LBF}: In LBF, agents should collect food items that are scattered randomly in a grid-world. Agents and items are assigned levels such that a group of agents can collect an item only if the sum of their levels is greater or equal to the level of the item. The convention for environment name is \{grid\_size\}$\times$\{grid\_size\}-\{player count\}p-\{food locations\}f.

\textbf{RWARE}: The convention for environment name is \{grid-size\}-\{player count\}ag.

\textbf{SMAC}: MMM2 (a symmetric scenario where each team controls seven marines, two marauders, and one medivac unit), 3s5z (a symmetric scenario where each team controls three stalkers and five zerglings for a total of eight agents), and 3s\_vs\_5z (team of three stalkers is controlled by agents to fight against a team of five game-controlled zerglings). For SMAC experiments, we only consider 5 learning rates due to its higher computational requirement.

Some important experimental details are listed below: 
\begin{itemize}
    \item For an IPPO agent, we change the learning rate of both the actor and the critic. 
    \item In IPPO, critic and actor updates might be done with different frequencies. In such a case switching is done based on critic training steps.
    \item We use the Adam~\citep{kingma2014adam} optimizer in all experiments (we only change the learning rate hyperparameter). Even though Adam adaptively changes the gradient signal, we can still control its scale with changing the learning rate coefficient directly.
    \item In the case of IPPO, we change the learning rates of both the actor and critic.
    \item d) In the case of two agents, each agent learns with the respective learning rate while in the case of more than two agents, one agent learns with $lr_0$ and the rest with $lr_1$. 
\end{itemize}

\begin{table}[ht]
    \centering
    \small
    \caption{Hyperparameters for IPPO without parameter sharing.}
    \begin{tabular}{l cccccc}
    \\ \hline \\
     & {} &  MPE & SMAC & LBF & RWARE \\
    \\ \hline \\
        &\begin{tabular}{@{}c@{}}Hidden \\ dimension\end{tabular} & 128 & 64 & 128 &128\\
        &\begin{tabular}{@{}c@{}}Reward \\ standardisation\end{tabular}  & True & True & False &False \\
        &Network type  & FC & FC & GRU & FC\\
        &\begin{tabular}{@{}c@{}}Entropy \\ coefficient\end{tabular} & 0.01 & 0.001 & 0.001 &0.001 \\
        &Target update & \begin{tabular}{@{}c@{}}0.01 \\ (soft)\end{tabular} & \begin{tabular}{@{}c@{}}0.01 \\ (soft)\end{tabular} & \begin{tabular}{@{}c@{}}200 \\ (hard)\end{tabular} & \begin{tabular}{@{}c@{}}0.01 \\ (soft)\end{tabular} \\
        &n-step  & 10 & 10 & 5 & 10 \\
    \\ \hline \\
    \end{tabular}
    \label{tab:hyper_ippo_ns}
\end{table}

\begin{table}[ht]
\vspace{-10mm}
    \centering
    \small
    \caption{ Learning rates for MTPPO without parameter sharing.}
    \begin{tabular}{l ccccc}
    \\ \hline \\
     Environment &  Learning rates \\
    \\ \hline \\
        MPE & \begin{tabular}{@{}c@{}}$\{1.25\times10^{-5}, 2.5\times10^{-5}, 5\times10^{-5}, 1\times10^{-4},$\\ $ 2\times10^{-4}, 4\times10^{-4}, 8\times10^{-4}\}$\end{tabular}
        \\
        LBF & \begin{tabular}{@{}c@{}}$\{1.25\times10^{-5}, 2.5\times10^{-5}, 5\times10^{-5}, 1\times10^{-4},$\\ $ 2\times10^{-4}, 4\times10^{-4}, 8\times10^{-4}\}$\end{tabular}
        \\ 
        RWARE & \begin{tabular}{@{}c@{}}$\{6.25\times10^{-5}, 1.25\times10^{-4}, 2.5\times10^{-4}, 5\times10^{-4},$\\ $ 1\times10^{-3}, 2\times10^{-3}, 4\times10^{-3}\} $\end{tabular}
        \\
        SMAC & \begin{tabular}{@{}c@{}}$\{1.25\times10^{-4}, 2.5\times10^{-4}, 5\times10^{-4},$\\ $ 1\times10^{-3}, 2\times10^{-3}\}$\end{tabular}
        \\
    \hline \\
    \end{tabular}
    \label{tab:lr_ippo_ns}
\end{table}

\begin{table}[ht]
    \centering
    \footnotesize
    \caption{Hyperparameters for IQL without parameter sharing.}
    \begin{tabular}{l cccccc}
    \\ \hline \\
     &{} &  MPE & SMAC & LBF & RWARE \\
    \\ \hline \\
        &\begin{tabular}{@{}c@{}}Hidden  dimension\end{tabular} & 128 & 64 & 64 &64\\
        &\begin{tabular}{@{}c@{}}Reward standardisation\end{tabular}  & True & True & True &True \\
        &\begin{tabular}{@{}c@{}}Network  type\end{tabular}   & FC & GRU & GRU & FC\\
        &Target update &\begin{tabular}{@{}c@{}}0.01  (soft)\end{tabular} &\begin{tabular}{@{}c@{}}200 (hard)\end{tabular} & \begin{tabular}{@{}c@{}}200 (hard)\end{tabular} & \begin{tabular}{@{}c@{}}0.01 (soft)\end{tabular} \\
    \\ \hline \\
    \end{tabular}
    \label{tab:hyper_iql_ns}
\end{table}

\begin{table}[ht]
\vspace{-5mm}
    \centering
    \small
    \caption{Learning rate for MTQL without parameter sharing.}
    \begin{tabular}{l ccccc}
    \\ \hline \\
     Environment &  Learning rates \\
    \\ \hline \\
        MPE & \begin{tabular}{@{}c@{}}$\{6.25\times10^{-5}, 1.25\times10^{-4}, 2.5\times10^{-4}, 5\times10^{-4}, $ \\ $ 1\times10^{-3}, 2\times10^{-3}, 4\times10^{-3} \}$\end{tabular}
        \\
        LBF & \begin{tabular}{@{}c@{}}$\{3.7\times10^{-5}, 7.5\times10^{-5}, 1.5\times10^{-4}, 3\times10^{-4}, $ \\
        $ 6\times10^{-4}, 1.2\times10^{-3}, 2.4\times10^{-3}\}$\end{tabular}
        \\ 
        RWARE & \begin{tabular}{@{}c@{}}$\{6.25\times10^{-5}, 1.25\times10^{-4}, 2.5\times10^{-4}, 5\times10^{-4}, $ \\
        $ 1\times10^{-3}, 2\times10^{-3}, 4\times10^{-3}\}$\end{tabular}
        \\
        SMAC & \begin{tabular}{@{}c@{}}$\{\{1.25\times10^{-4}, 2.5\times10^{-4}, 5\times10^{-4}, $ $ 1\times10^{-3}, 2\times10^{-3}\}$\end{tabular}
        \\
    \hline \\
    \end{tabular}
    \label{tab:lr_iql_ns}
\end{table}

\newpage
\vspace{-5mm}
\section{More results} \label{app:learning_curves_heatmaps}

Regarding the experimental setup, 5 seeds may not be enough due to high variance in the performance of the baseline algorithms in certain tasks. We did the following to make sure that improvements due to using multi-timscale learning is indeed significant: For environments with high variance in the performance, we ran the experiments for 25 seeds and the results are reported in~\ref{app:25seeds_results}. Compared with the original results with 5 seeds, we can see that the improvements due to multi-timescale learning are statistically significant indeed.

\begin{table}[h!]
    \vspace{-5mm}
    \vskip -\baselineskip
    \centering
    \caption{Rerunning experiments with 25 seeds for the tasks with high variance. The improvements due to multi-timescale learning are still statistically significant.}
    \begin{tabular}{@{}ccccc@{}}
    \toprule
        &  Foraging-10*10-3p-3f & SimpleTag & RWARE tiny4ag \\
    \midrule
        IQL & $0.42 (\pm 0.01) $ & $57.36 (\pm1.57)$ & $0.1243 (\pm0.0158)$ \\
         MTQL & $\textbf{0.48} (\pm 0.02)$ & $\textbf{64.03} (\pm2.00)$ & $\textbf{0.1405} (\pm0.0172)$ \\
    \bottomrule
    \end{tabular}
    \vskip -\baselineskip
    \label{app:25seeds_results}
\end{table}

\subsection{Learning curves and Heatmaps}

In this section, we provided the best learning curves for MTPPO and MTQL in Figures~\ref{app:fig:best-learning-curves-sippo} and \ref{app:fig:best-learning-curves-siql}. We also included the learning curves across switching periods in Figures~\ref{app:all-learning-curves-ippo} and \ref{app:all-learning-curves-iql}.

\begin{figure}[ht]
\begin{center}
\vspace{-2mm}
\centerline{\includegraphics[width=0.85\columnwidth]{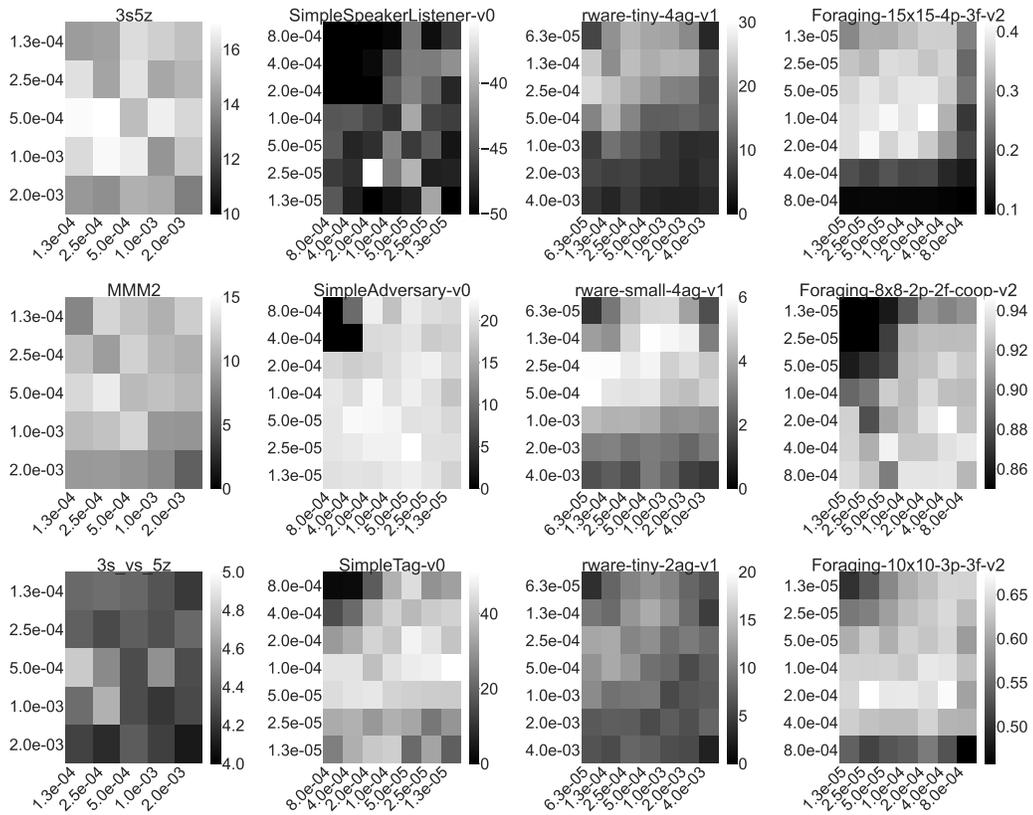}}
\caption{\em Final performance of IPPO with different learning rate combinations. These heatmaps are for the best switching period values. It's clear that in many tasks, non-diagonal values (MTPPO) have relatively better performance compared to diagonal values (IPPO).}
\label{app:best-heatmaps-sippo}
\end{center}
\end{figure}

\begin{figure}[ht]
\begin{center}
\centerline{\includegraphics[width=0.95\columnwidth]{figs/test_curves/ippo_staged_ns_curves.pdf}}
\caption{\em Learning curves for each task. MTPPO leads to faster convergence than IPPO in many tasks. Solid lines are mean test returns over 100 test episodes averaged over 5 independent seeds. Shaded regions indicates the standard-error. Smoothing with window size = 5 is used.}
\label{app:fig:best-learning-curves-sippo}
\end{center}
\end{figure}

\begin{figure}[ht]
\begin{center}
\centerline{\includegraphics[width=0.85\columnwidth]{figs/test_curves/iql_staged_ns_curves.pdf}}
\caption{\em Learning curves for each task. MTQL leads to faster convergence than IQL in many tasks. Solid lines are mean test returns over 100 test episodes averaged over 5 independent seeds. Shaded regions indicates the standard-error. Smoothing with window size = 5 is used.}
\label{app:fig:best-learning-curves-siql}
\end{center}
\end{figure}

We also provide heatmaps of the final performance for MTPPO for all combination of learning rates in Figure~\ref{app:best-heatmaps-sippo}.

\begin{figure*}[ht]
\centering
\begin{subfigure}[Foraging-8x8-2p-2f-coop-v2]{
    \centering
    \includegraphics[ width=0.4\linewidth]{figs/test_curves/test-return-Foraging-8x8-2p-2f-coop-v2-ippo_staged_ns-metric-iqm-switch-period10000.0-curves.pdf}
    }
\end{subfigure}%
\begin{subfigure}[Foraging-10x10-3p-3f-v2]{
    \centering
    \includegraphics[ width=0.4\linewidth]{figs/test_curves/test-return-Foraging-10x10-3p-3f-v2-ippo_staged_ns-metric-iqm-switch-period10000.0-curves.pdf}
    }
\end{subfigure}%
\begin{subfigure}[Foraging-15x15-4p-3f-v2]{
    \centering
    \includegraphics[ width=0.4\linewidth]{figs/test_curves/test-return-Foraging-15x15-4p-3f-v2-ippo_staged_ns-metric-iqm-switch-period10000.0-curves.pdf}
    }
\end{subfigure}%
\begin{subfigure}[rware-tiny-2ag-v1]{
    \centering
    \includegraphics[ width=0.4\linewidth]{figs/test_curves/test-return-rware-tiny-2ag-v1-ippo_staged_ns-metric-iqm-switch-period10000.0-curves.pdf}
    }
\end{subfigure}%
\begin{subfigure}[rware-tiny-4ag-v1]{
    \centering
    \includegraphics[ width=0.4\linewidth]{figs/test_curves/test-return-rware-tiny-4ag-v1-ippo_staged_ns-metric-iqm-switch-period10000.0-curves.pdf}
    }
\end{subfigure}%
\begin{subfigure}[rware-small-4ag-v1]{
    \centering
    \includegraphics[ width=0.4\linewidth]{figs/test_curves/test-return-rware-small-4ag-v1-ippo_staged_ns-metric-iqm-switch-period10000.0-curves.pdf}
    }
\end{subfigure}%
\begin{subfigure}[SimpleSpeakerListener-v0]{
    \centering
    \includegraphics[ width=0.4\linewidth]{figs/test_curves/test-return-SimpleSpeakerListener-v0-ippo_staged_ns-metric-iqm-switch-period10000.0-curves.pdf}
    }
\end{subfigure}%
\begin{subfigure}[SimpleTag-v0]{
    \centering
    \includegraphics[ width=0.4\linewidth]{figs/test_curves/test-return-SimpleTag-v0-ippo_staged_ns-metric-iqm-switch-period10000.0-curves.pdf}
    }
\end{subfigure}%
\begin{subfigure}[SimpleAdversary-v0]{
    \centering
    \includegraphics[ width=0.4\linewidth]{figs/test_curves/test-return-SimpleAdversary-v0-ippo_staged_ns-metric-iqm-switch-period10000.0-curves.pdf}
    }
\end{subfigure}%
\begin{subfigure}[MMM2]{
    \centering
    \includegraphics[ width=0.4\linewidth]{figs/test_curves/test-return-MMM2-ippo_staged_ns-metric-iqm-switch-period10000.0-curves.pdf}
    }
\end{subfigure}%
\begin{subfigure}[3s5z]{
    \centering
    \includegraphics[ width=0.4\linewidth]{figs/test_curves/test-return-3s5z-ippo_staged_ns-metric-iqm-switch-period10000.0-curves.pdf}
    }
\end{subfigure}%
\begin{subfigure}[3s vs 5z]{
    \centering
    \includegraphics[ width=0.4\linewidth]{figs/test_curves/test-return-3s_vs_5z-ippo_staged_ns-metric-iqm-switch-period10000.0-curves.pdf}
    }
\end{subfigure}%
\caption{\em Learning curves of MTPPO and IPPO for each task and different switching periods. Solid lines are mean test returns over 100 test episodes averaged over 5 independent seeds. Shadow region indicates the standard-error. Smoothing with window size = 5 is used. Differnece in IPPO's performance across different switching periods is due to the variance in the results.}
\label{app:all-learning-curves-ippo} 
\end{figure*}

\begin{figure*}[ht]
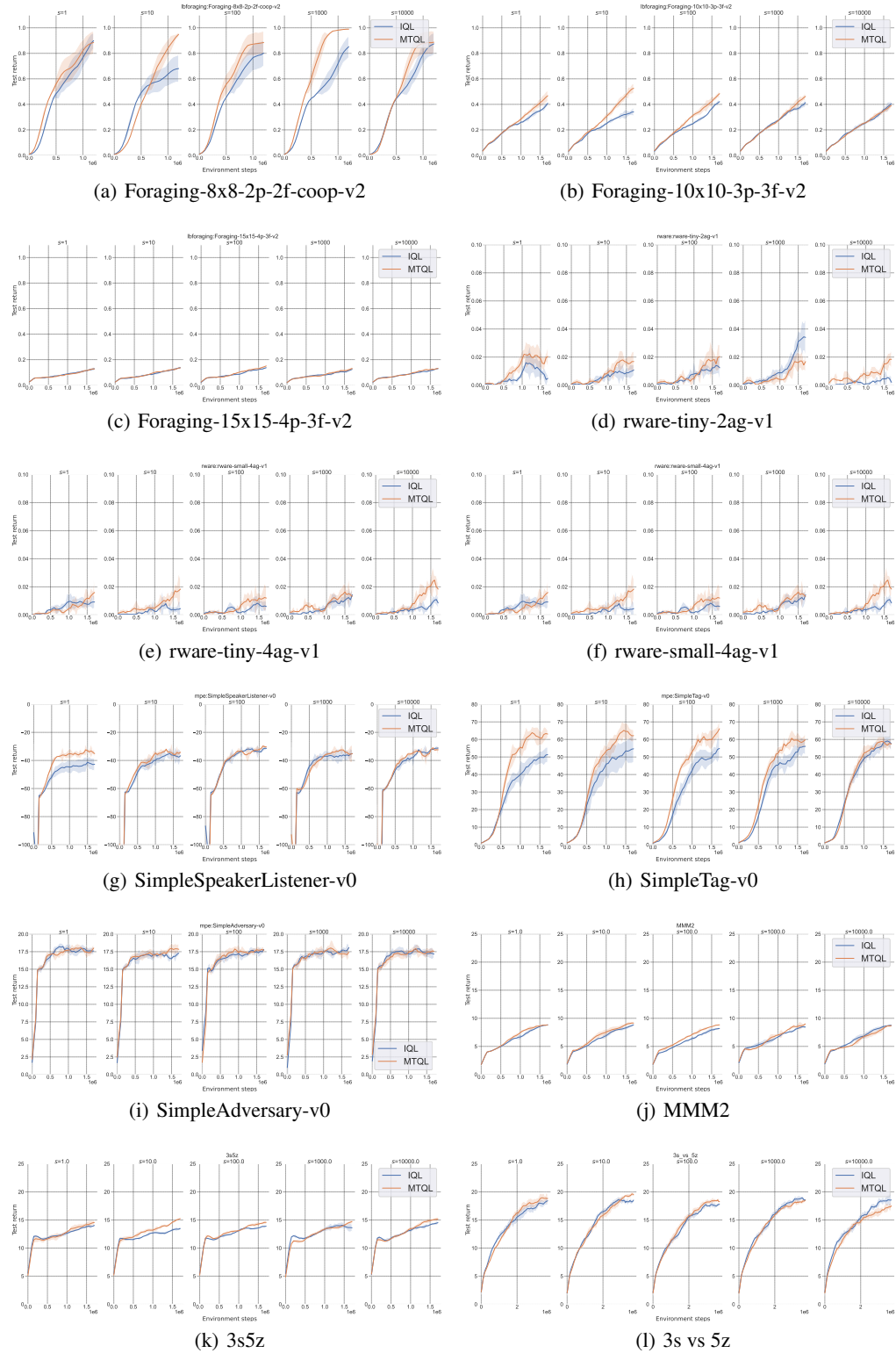

\centering
\begin{subfigure}[Foraging-8x8-2p-2f-coop-v2]{
    \centering
    \includegraphics[ width=0.4\linewidth]{figs/test_curves/test-return-Foraging-8x8-2p-2f-coop-v2-iql_staged_ns-metric-iqm-switch-period10000-curves.pdf}
    }
\end{subfigure}%
\begin{subfigure}[Foraging-10x10-3p-3f-v2]{
    \centering
    \includegraphics[ width=0.4\linewidth]{figs/test_curves/test-return-Foraging-10x10-3p-3f-v2-iql_staged_ns-metric-iqm-switch-period10000-curves.pdf}
    }
\end{subfigure}%
\begin{subfigure}[Foraging-15x15-4p-3f-v2]{
    \centering
    \includegraphics[ width=0.4\linewidth]{figs/test_curves/test-return-Foraging-15x15-4p-3f-v2-iql_staged_ns-metric-iqm-switch-period10000-curves.pdf}
    }
\end{subfigure}%
\begin{subfigure}[rware-tiny-2ag-v1]{
    \centering
    \includegraphics[ width=0.4\linewidth]{figs/test_curves/test-return-rware-tiny-2ag-v1-iql_staged_ns-metric-iqm-switch-period10000-curves.pdf}
    }
\end{subfigure}%
\begin{subfigure}[rware-tiny-4ag-v1]{
    \centering
    \includegraphics[ width=0.4\linewidth]{figs/test_curves/test-return-rware-small-4ag-v1-iql_staged_ns-metric-iqm-switch-period10000-curves.pdf}
    }
\end{subfigure}%
\begin{subfigure}[rware-small-4ag-v1]{
    \centering
    \includegraphics[ width=0.4\linewidth]{figs/test_curves/test-return-rware-small-4ag-v1-iql_staged_ns-metric-iqm-switch-period10000-curves.pdf}
    }
\end{subfigure}%
\begin{subfigure}[SimpleSpeakerListener-v0]{
    \centering
    \includegraphics[ width=0.4\linewidth]{figs/test_curves/test-return-SimpleSpeakerListener-v0-iql_staged_ns-metric-iqm-switch-period10000-curves.pdf}
    }
\end{subfigure}%
\begin{subfigure}[SimpleTag-v0]{
    \centering
    \includegraphics[ width=0.4\linewidth]{figs/test_curves/test-return-SimpleTag-v0-iql_staged_ns-metric-iqm-switch-period10000-curves.pdf}
    }
\end{subfigure}%
\begin{subfigure}[SimpleAdversary-v0]{
    \centering
    \includegraphics[ width=0.4\linewidth]{figs/test_curves/test-return-SimpleAdversary-v0-iql_staged_ns-metric-iqm-switch-period10000-curves.pdf}
    }
\end{subfigure}%
\begin{subfigure}[MMM2]{
    \centering
    \includegraphics[ width=0.4\linewidth]{figs/test_curves/test-return-MMM2-iql_staged_ns-metric-mean-switch-period10000.0-curves.pdf}
    }
\end{subfigure}%
\begin{subfigure}[3s5z]{
    \centering
    \includegraphics[ width=0.4\linewidth]{figs/test_curves/test-return-3s5z-iql_staged_ns-metric-mean-switch-period10000.0-curves.pdf}
    }
\end{subfigure}%
\begin{subfigure}[3s vs 5z]{
    \centering
    \includegraphics[ width=0.4\linewidth]{figs/test_curves/test-return-3s_vs_5z-iql_staged_ns-metric-mean-switch-period10000.0-curves.pdf}
    }
\end{subfigure}%
\caption{\em Learning curves of MTQL and IQL for each task and different switching periods. Solid lines are mean test returns over 100 test episodes averaged over 5 independent seeds. Shadow region indicates the standard-error. Smoothing with window size = 5 is used. Differnece in IQL's performance across different switching periods is due to the variance in the results.}
\label{app:all-learning-curves-iql} 
\end{figure*}

\clearpage

\subsection{Orion Experiments}\label{app:orion_exp}

For orion experiments, all the hyperparameters are similar to table \ref{tab:hyper_iql_ns}. However, learning rates and switching periods are sampled from distributions mentioned in table \ref{tab:orion_hyp}. 


\begin{figure}[ht]
\centering
\begin{subfigure}[Foraging-8x8-2p-2f-coop]{
    \centering
    \includegraphics[ width=0.3\linewidth]{figs/orion_results/SIQL_Orion_Foraging-8-2p-2f-coop-v2_regret.pdf}
    }
    \label{fig:orionlbf8}
\end{subfigure}%
\begin{subfigure}[Foraging-10x10-3p-3f]{
    \centering
    \includegraphics[ width=0.3\linewidth]{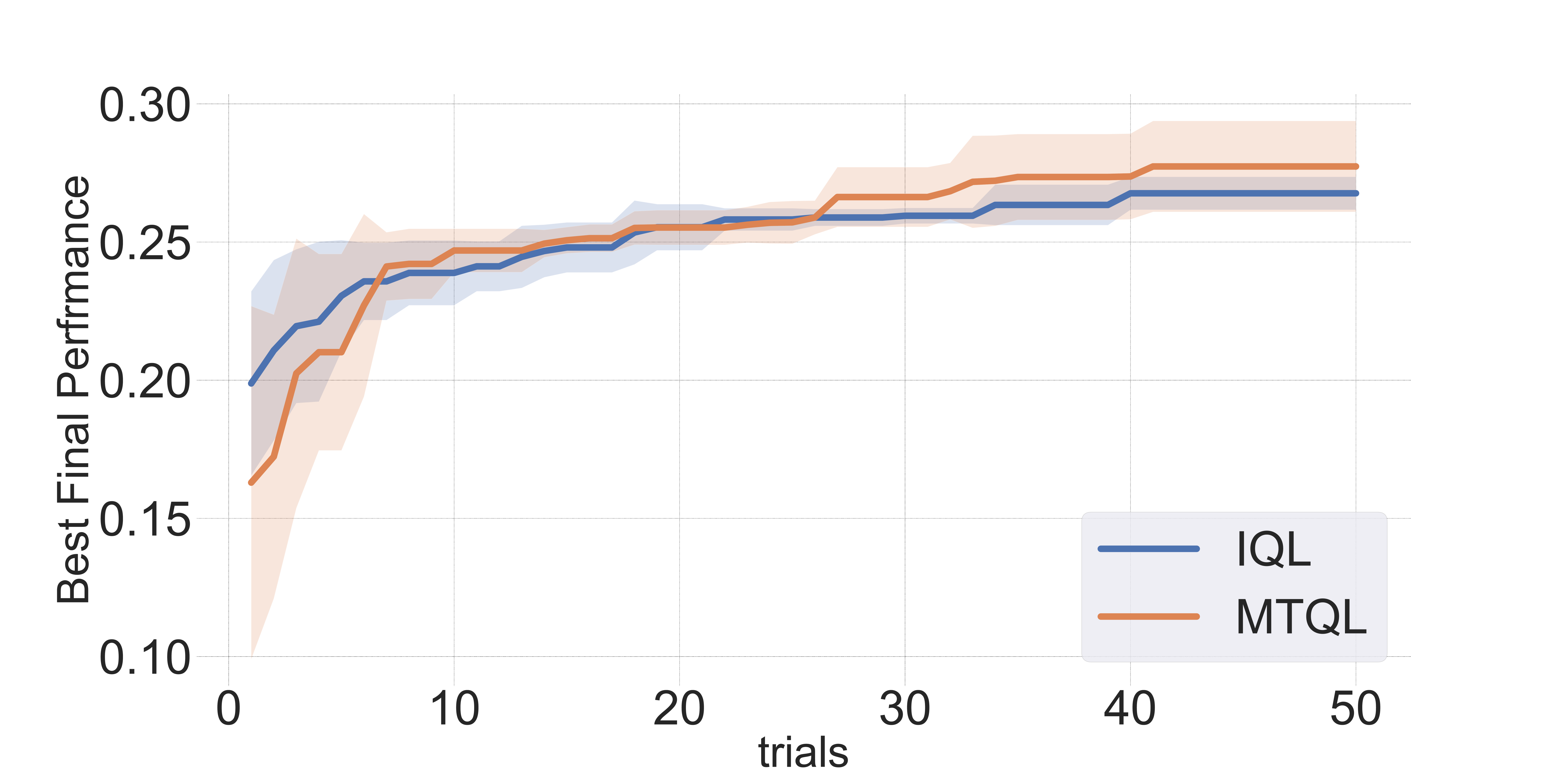}
    }
    \label{fig:orionlbf10}
\end{subfigure}%
\begin{subfigure}[Foraging-15x15-4p-3f]{
    \centering
    \includegraphics[ width=0.3\linewidth]{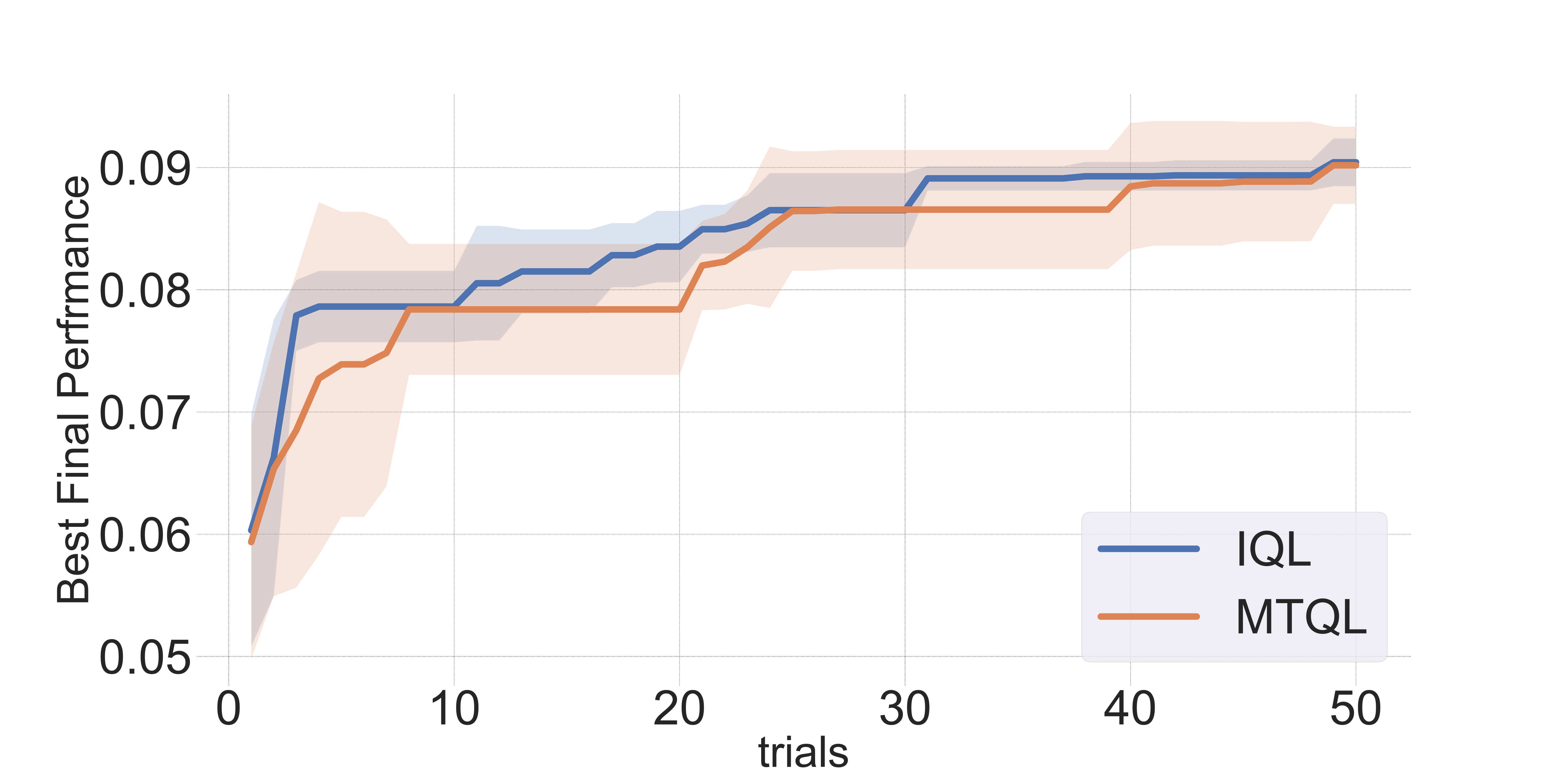}
    }
    \label{fig:orionlbf15}
\end{subfigure}%
\begin{subfigure}[mpe:SimpleSpeakerListener]{
    \centering
    \includegraphics[ width=0.3\linewidth]{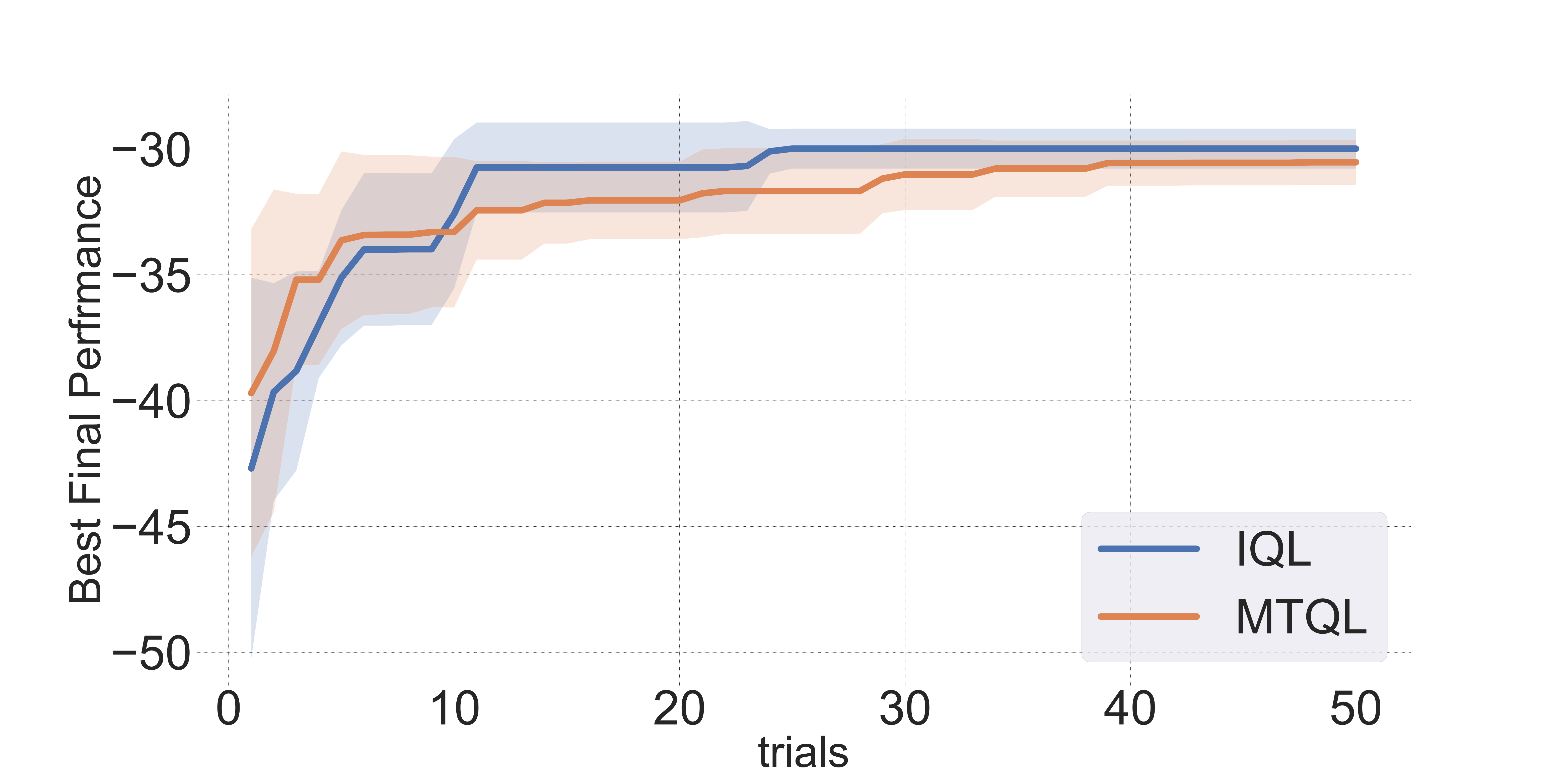}
    }
    \label{fig:orion_Speaker}
\end{subfigure}%
\begin{subfigure}[mpe:SimpleAdversary]{
    \centering
    \includegraphics[ width=0.3\linewidth]{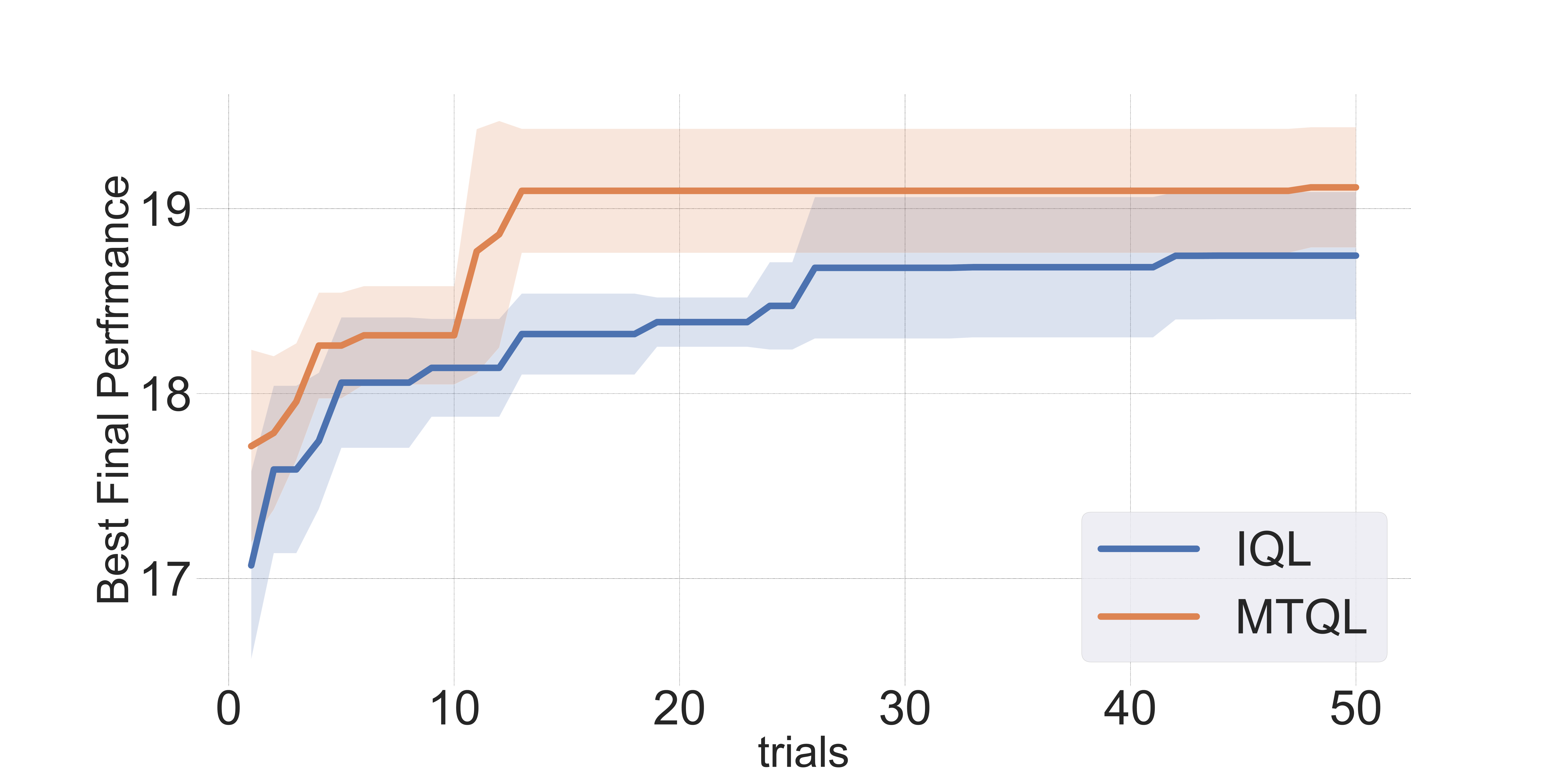}
    }
    \label{fig:orion_Adversary}
\end{subfigure}%
\begin{subfigure}[mpe:SimpleTag]{
    \centering
    \includegraphics[ width=0.3\linewidth]{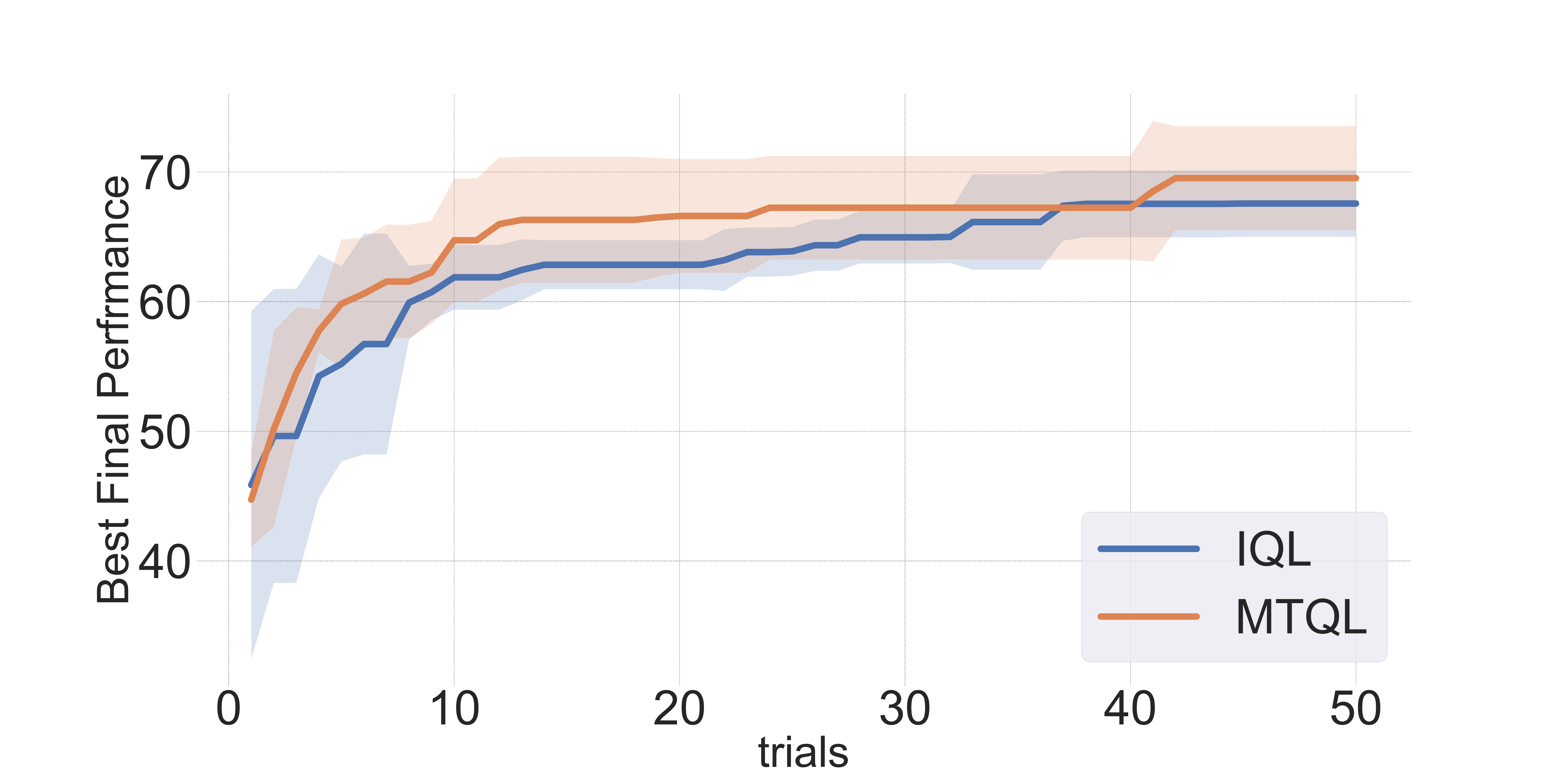}
    }
    \label{fig:orion_Tag}
\end{subfigure}%
\begin{subfigure}[smac:3s vs 5z]{
    \centering
    \includegraphics[ width=0.3\linewidth]{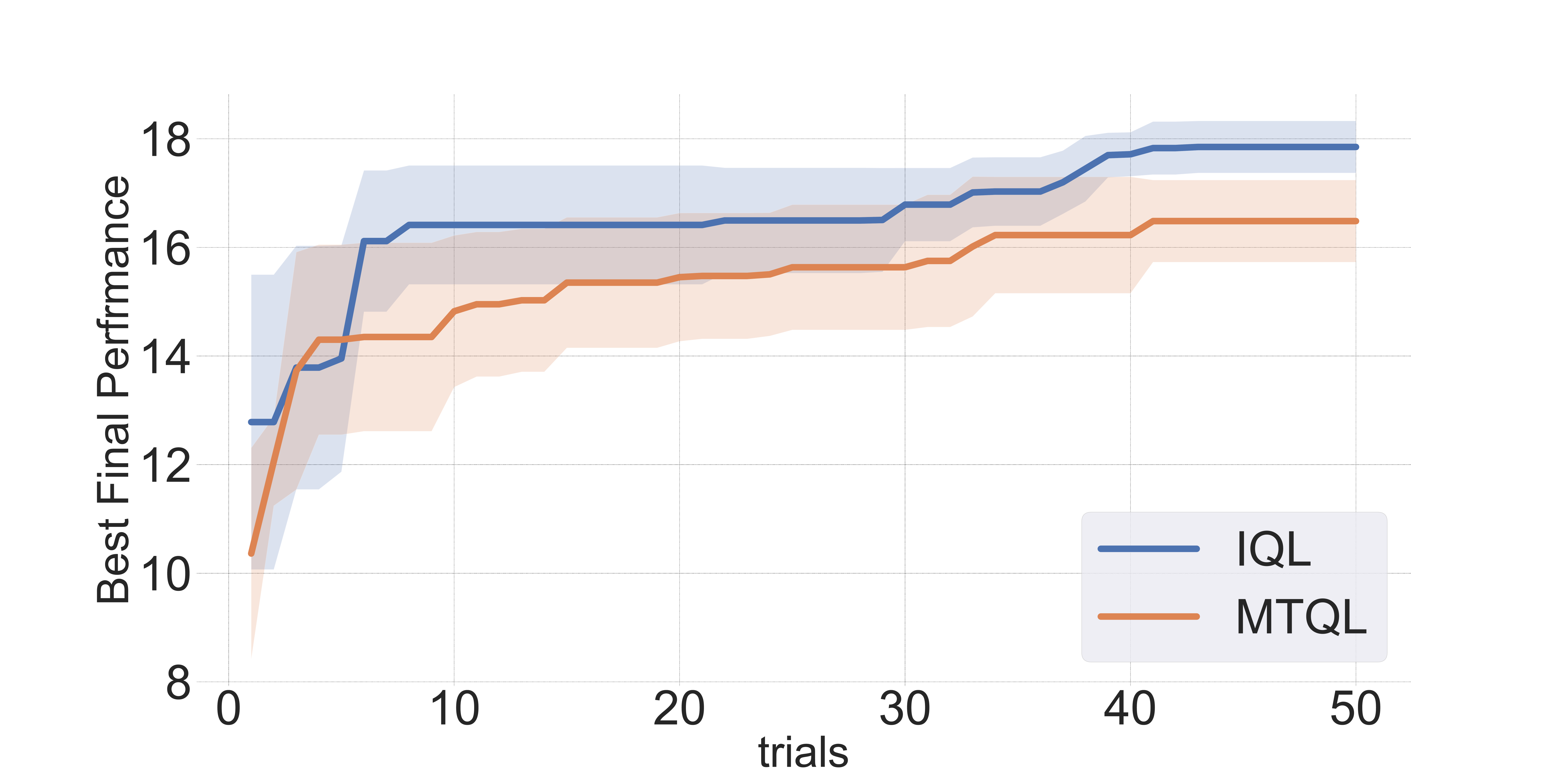}
    }
    \label{fig:orion_3s_vs_5z}
\end{subfigure}%
\begin{subfigure}[smac:3s5z]{
    \centering
    \includegraphics[ width=0.3\linewidth]{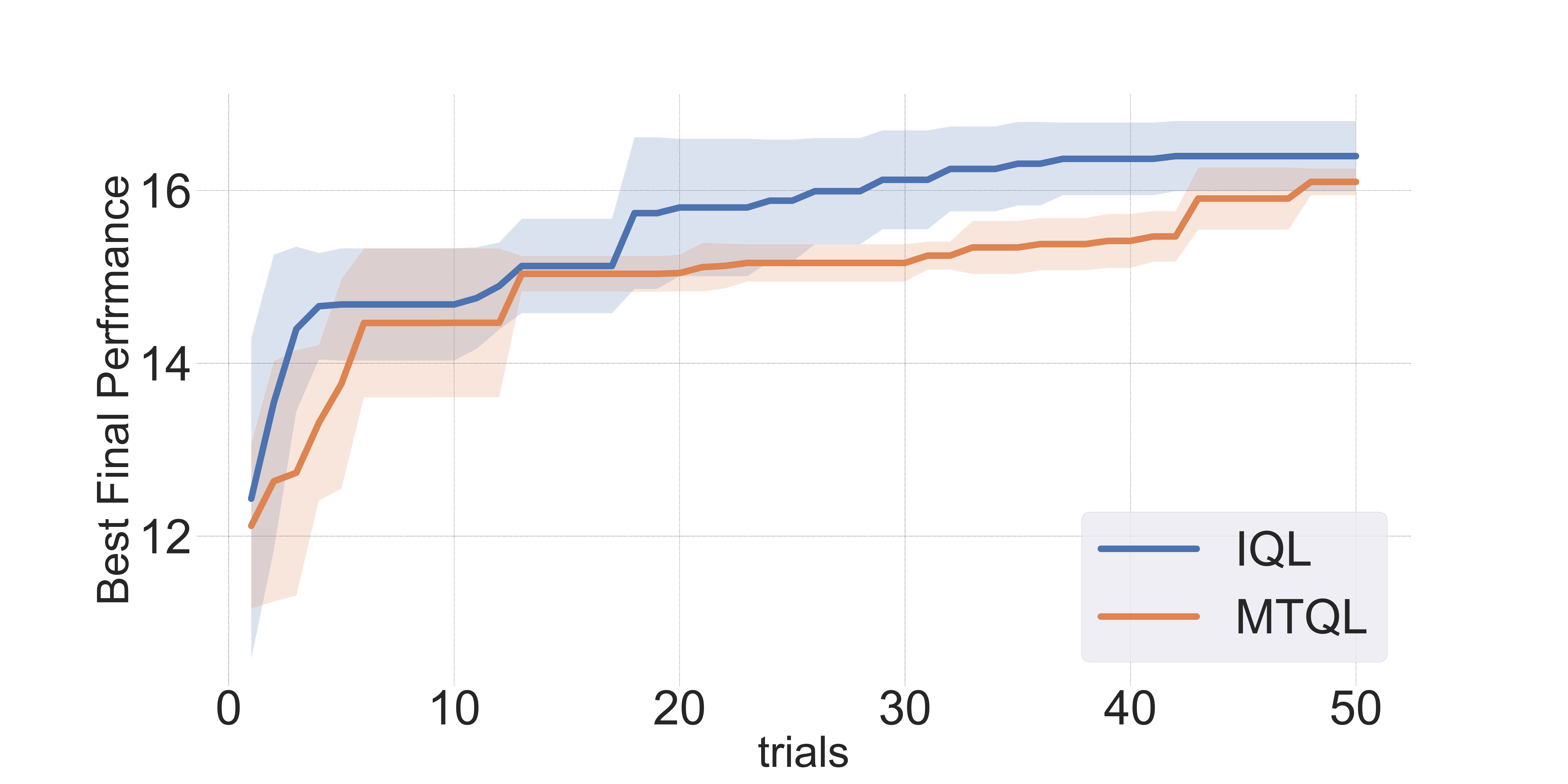}
    }
    \label{fig:orion_3s5z}
\end{subfigure}%
\begin{subfigure}[smac:MMM2]{
    \centering
    \includegraphics[ width=0.3\linewidth]{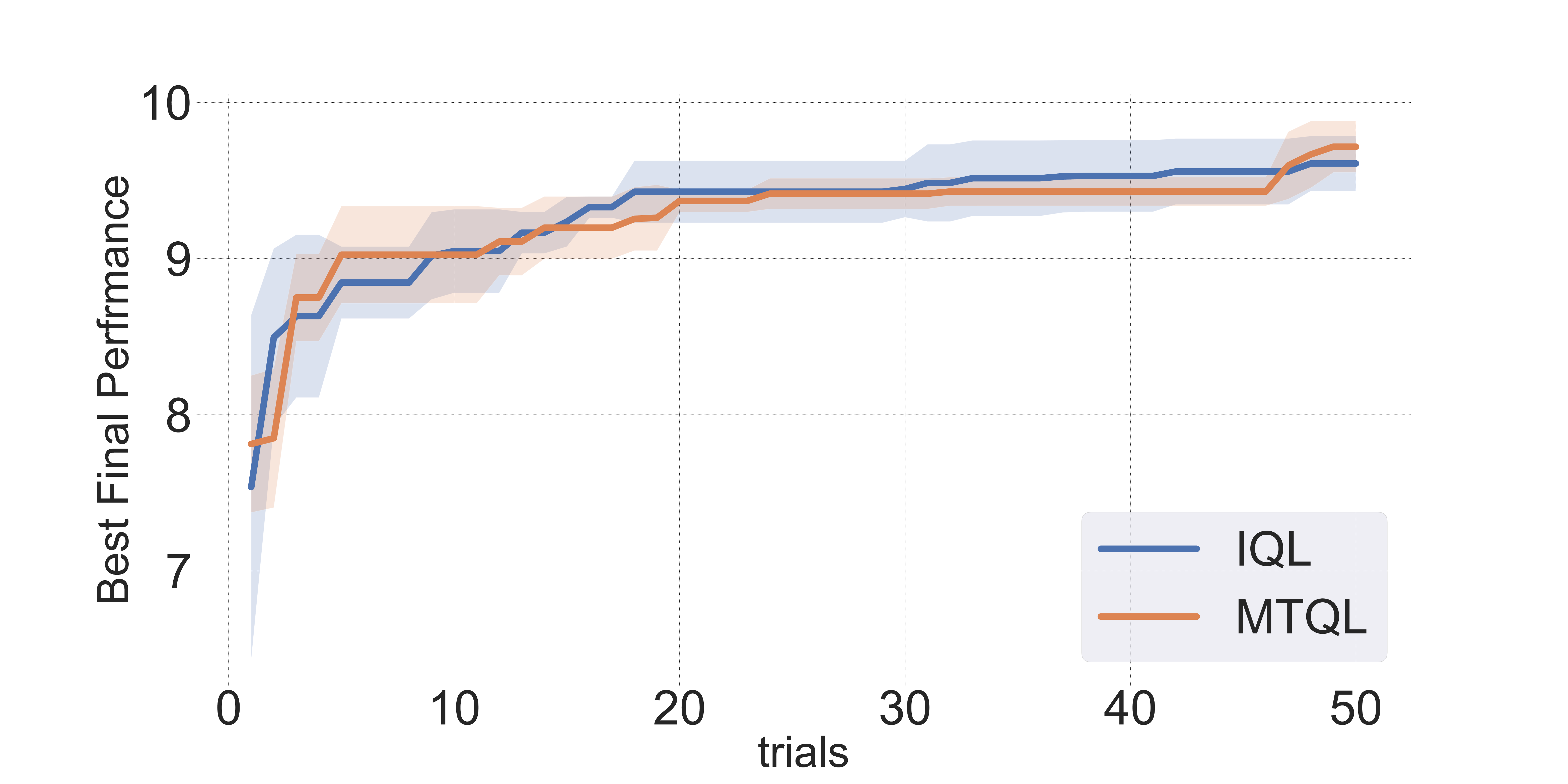}
    }
    \label{fig:orion_MMM2}
\end{subfigure}%
\begin{subfigure}[rware-tiny-4ag]{
    \centering
    \includegraphics[ width=0.3\linewidth]{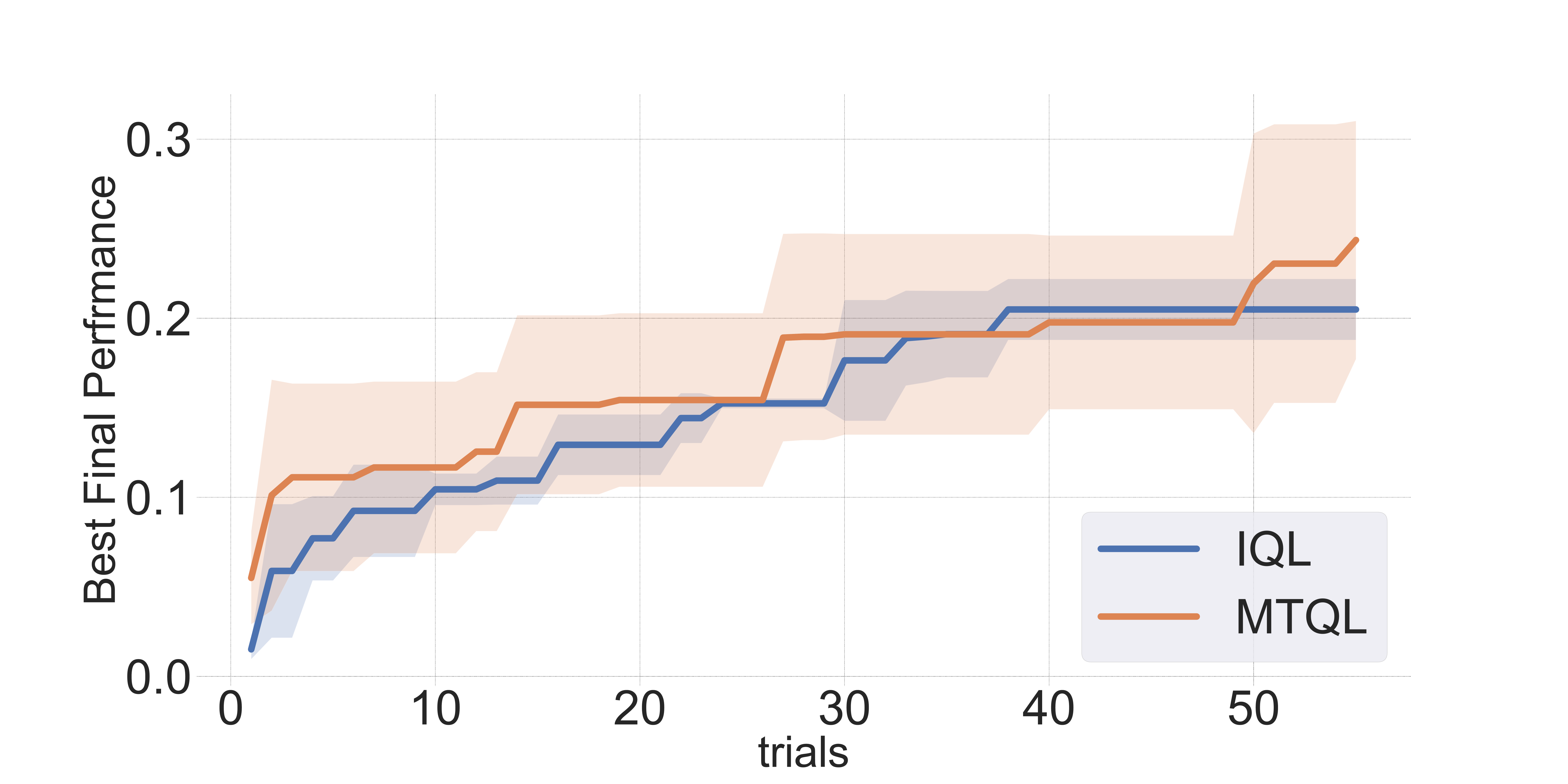}
    }
    \label{fig:orion_tiny4}
\end{subfigure}%
\caption{\em Orion optimization curves of MTQL and IQL for each task. Solid lines are the best final test return found by TPE algorithm averaged over 5 independent seeds. Shadow region indicates the standard-error.}
\label{fig:orion_all_iql} 
\end{figure}

\newpage

\begin{table}[ht]
    \centering
    \caption{ Orion hyper parameters.}
    \begin{tabular}{l ccc}
    \\ \hline \\
     Environment &  Learning rates & Switching period \\
    \\ \hline \\
        MPE & $loguniform(1e-05, 0.005)$ & $[10, 1000, 100000]$
        \\
        LBF & $~loguniform(5e-06, 0.005)$ & $[10, 1000, 100000]$
        \\
        RWARE & $loguniform(1e-05, 0.005)$ & $[10, 1000, 100000]$
        \\
    \hline \\
    \end{tabular}
    \label{tab:orion_hyp}
\end{table}

\begin{figure}[t]
\centering
\begin{subfigure}[MTPPO]{
    \centering
    \includegraphics[ width=0.45\linewidth]{figs/final_performance/ippo_staged_ns_perf.pdf}
    }
\end{subfigure}%
\begin{subfigure}[MTQL]{
    \centering
    \includegraphics[ width=0.45\linewidth]{figs/final_performance/iql_staged_ns_perf.pdf}
    }
\end{subfigure}%
\caption{\em Performance of IPPO, IQL, and their Multi-timescale version vs switching period for each task. At each switching period, the best-performing multi-timescale learning with different learning rates is compared with the best-performing independent learning with the same learning rates. Error bars represent the standard error over 5 seeds.} 
\label{app:perf-gain-errorbars} 
\vspace{-5mm}
\end{figure}


\end{document}